%% file: main.tex
\documentclass{article}

\usepackage[
  margin=1in,
  footskip=30pt
 ]{geometry}
 
\input{math_commands.tex}
\usepackage{mathtools}

\usepackage[pdftex]{graphicx}
\usepackage{hyperref}
\usepackage[round,sort]{natbib}
\usepackage{xcolor}
\usepackage{url}
\usepackage{soul}
\allowdisplaybreaks

\title{Inductive Biases and Variable Creation\\
in Self-Attention Mechanisms}

\author{Benjamin L. Edelman$^1$ \qquad Surbhi Goel$^2$\qquad Sham Kakade$^{1,2}$\qquad Cyril Zhang$^2$\\
\vspace{-2mm} \\
  \normalsize{$^1$Harvard University\quad $^2$Microsoft Research NYC}\\
\normalsize{ \texttt{bedelman@g.harvard.edu, \{goel.surbhi, sham.kakade, cyrilzhang\}@microsoft.com} } }
\date{}

\begin{document}

\maketitle
\input{0-abstract}
\input{1-intro}
\input{2-prelims}
\input{3-theory}
\input{4-experiments}

\input{5-conclusion}

\bibliography{bib}
\bibliographystyle{plainnat}

\appendix
\input{A-capacity-proofs}
\input{B-examples-proofs}

\input{C-full-experiments}
\input{D-related}

\end{document}

%% file: math_commands.tex
\usepackage{amsmath,amsfonts,amssymb}
\usepackage{mathtools}
\usepackage{amsthm}
\usepackage{enumitem}

\def\eqref#1{equation~\ref{#1}}

\def\1{\bm{1}}

\def\eps{{\epsilon}}

\DeclareMathAlphabet{\mathsfit}{\encodingdefault}{\sfdefault}{m}{sl}
\SetMathAlphabet{\mathsfit}{bold}{\encodingdefault}{\sfdefault}{bx}{n}

\def\gD{{\mathcal{D}}}

\def\gF{{\mathcal{F}}}

\def\gI{{\mathcal{I}}}

\def\gN{{\mathcal{N}}}

\def\gV{{\mathcal{V}}}

\def\gX{{\mathcal{X}}}
\def\gY{{\mathcal{Y}}}
\def\gZ{{\mathcal{Z}}}

\newcommand{\R}{\mathbb{R}}

\newcommand{\softmax}{\mathsf{softmax}}

\DeclareMathOperator{\sign}{sign}

\usepackage{amsmath,amsthm,amssymb}

\newtheorem{theorem}{Theorem}[section]
\newtheorem{proposition}[theorem]{Proposition}
\newtheorem{corollary}[theorem]{Corollary}
\newtheorem{lemma}[theorem]{Lemma}
\newtheorem{definition}[theorem]{Definition}
\newtheorem{assumption}[theorem]{Assumption}
\usepackage{bbm}
\renewcommand{\hat}{\widehat}
\newcommand{\RR}{\mathbb{R}}
\renewcommand{\epsilon}{\varepsilon}
\newcommand{\pa}[1]{\left(#1\right)}
\newcommand{\bra}[1]{\left[#1\right]}

\newcommand{\norm}[1]{\left\|#1\right\|}

\newcommand{\CLS}{\textup{\texttt{[CLS]}}}
\newcommand{\tfHead}{\mathsf{tf{\text-}head}}
\newcommand{\tfLayer}{\mathsf{tf{\text-}layer}}
\newcommand{\tfScalar}{\mathsf{tf{\text-}scalar}}
\newcommand{\mlp}{\mathsf{mlp}}
\newcommand{\tfmlp}{\mathsf{tf+mlp}}
\newcommand{\tfMulti}{{\mathsf{tf}\text{-}\mathsf{heads}}}
\newcommand{\tfHeadUnbounded}{{\mathsf{tf}\text{-}\mathsf{head}\text{-}\mathsf{unbounded}}}

%% file: 0-abstract.tex
\begin{abstract}
Self-attention, an architectural motif designed to model long-range interactions in sequential data, has driven numerous recent breakthroughs in natural language processing and beyond. This work provides a theoretical analysis of the inductive biases of self-attention modules. Our focus is to rigorously establish which functions and long-range dependencies self-attention blocks prefer to represent. Our main result shows that bounded-norm Transformer networks ``create sparse variables'': a single self-attention head can represent a sparse function of the input sequence, with sample complexity scaling only logarithmically with the context length. To support our analysis, we present synthetic experiments to probe the sample complexity of learning sparse Boolean functions with Transformers.
\end{abstract}

%% file: 1-intro.tex
\section{Introduction}

Self-attention mechanisms have comprised an era-defining cornerstone of deep learning in recent years, appearing ubiquitously in
empirical breakthroughs in generative sequence modeling and unsupervised
representation learning. 
Starting with natural language
\citep{vaswani2017attention}, self-attention has enjoyed
surprising empirical successes in numerous and diverse modalities of data. 
In many of these settings, self-attention has supplanted
traditional recurrent and convolutional architectures, which are
understood to incorporate inductive biases about temporal and
translational invariances in the data. Self-attention models discard
these functional forms, in favor of directly and globally modeling
long-range interactions within the input context.\looseness=-1

The proliferation of self-attention raises a fundamental question about its inductive biases:
\emph{which functions do self-attention networks prefer to represent?}
Various intuitions and empirics inform the design of these architectures, but formal statistical abstractions and analyses are missing in this space.
To this end, this work initiates an analysis of the statistical
foundations of self-attention.

We identify an inductive bias for self-attention, for which we coin the term \emph{sparse variable creation}: a bounded-norm self-attention head learns a sparse function (which only depends on a small subset of input coordinates, such as a constant-fan-in gate in a Boolean circuit) of a length-$T$ context, with sample complexity scaling as $\log(T)$.
The main technical novelty in this work is a covering number-based capacity bound for attention mechanisms (including Transformer heads, as well as related and future architectures), implying norm-based generalization bounds. This is accompanied by a matching representational result, showing that bounded-norm self-attention heads are indeed capable of representing $s$-sparse functions with weight norms $2^{O(s)}$
(or $\mathrm{poly}(s)$, for symmetric sparse functions). This provides a theoretical account for why attention models can learn long-range dependencies without overfitting.

Finally, we conduct synthetic experiments to probe the sample efficiency of learning sparse interactions with self-attention. We train Transformer models to identify sparse Boolean functions with randomly chosen indices, and corroborate the sample complexity scaling law predicted by the theory. A variant of this experiment (with i.i.d. samples) reveals a \emph{computational} mystery, beyond the scope of our current statistical analysis: we find that Transformers can successfully learn the ``hardest'' (in the sense of SQ-dimension) $s$-sparse functions: the XOR (parity) functions.

\subsection{Related work}

The direct precursors to modern self-attention architectures were
recurrent and convolutional networks augmented with attention
mechanisms \citep{bahdanau2014neural,luong2015effective,xu2015show}.
Landmark work by \citet{vaswani2017attention} demonstrated
significantly improvements in machine translation via a pure self-attention
architecture; autoregressive language models
\citep{liu2018generating,radford2018improving,radford2019language,brown2020language},
and self-supervised representation learning via masked language modeling
\citep{devlin2018bert} followed shortly.

\paragraph{Norm-based capacity bounds for neural nets.} There is a vast body of
literature dedicated to establishing statistical guarantees for neural
networks, including VC-dimension and shattering bounds (dating back to \cite{anthony1999neural}). In recent years, classical norm-based generalization bounds have been established for various architectures \citep{bartlett2017spectrally, neyshabur2015norm, neyshabur2017pac, golowich2018size, long2019generalization,chen2019generalization} using covering-based arguments. \citet{jiang2019fantastic} provide an extensive empirical study of how well these bounds predict generalization in practice. Our work complements these results by establishing the first norm-based capacity analysis for attention models. Our main results rely on a novel reduction to the $\ell_\infty$ covering number bound for linear function classes given by \cite{zhang2002covering}.

\paragraph{Other theoretical lenses on attention.} Our work complements various
existing theoretical perspectives on attention-based models.
\citet{vuckovic2020mathematical} formulate a dynamical system
abstraction of attention layers, arriving at similar Lipschitz
constant calculations to ours (which are coarser-grained, since they
focus on contractivity and stability rather than finite-sample
statistical guarantees).
\citet{zhang2019adaptive,snell2021approximating} study idealizations of the optimization problem of learning self-attention heads.
\citet{wei2021statistically} propose a definition of \emph{statistically meaningful approximation} of function classes that ties statistical learnability with expressivity, and show that Boolean circuits can be \emph{SM-approximated} by Transformers with a sample complexity bound that depends mildly on circuit depth (rather than context size), using a margin amplification procedure. \citet{kim2021lipschitz} show that standard dot-product attention is not Lipschitz for an \emph{unbounded} input domain, whereas our paper shows that norm-based generalization bounds are attainable with a $\norm{\cdot}_{2,\infty}$-bounded input domain.

See Appendix~\ref{appendix-sec:related} for a broader survey of the literature on attention and self-attention networks.

%% file: 2-prelims.tex
\section{Background and notation}

Throughout this paper, the input $X := [x_1 x_2 \dots x_T]^\top \in \RR^{T \times d}$ to an attention module (a.k.a. the context) will be a length-$T$ sequence of embeddings $x_t \in \RR^d$; $m$ refers to the sample size (i.e. number of length-$T$ sequences in a dataset). $\|\cdot\|_2$ denotes the spectral norm for matrices, and $\|\cdot\|_{p,q}$ denotes the $(p,q)$ matrix norm where the $p$-norm is over columns and $q$-norm over rows. For vectors, $\| \cdot\|_p$ denotes the $\ell_p$ norm; we drop the subscript for the $\ell_2$ norm. $B$ is generally used to quantify bounds on norms of matrices and $L$ for Lipschitz constants. $\Delta^{n-1}$ denotes the simplex in dimension $n$, that is, $\Delta^{n-1} := \{x \in \RR^n: x \ge 0, \|x\|_1 = 1\}$.

\paragraph{Covering numbers.}
Our main technical contribution is a generalization bound arising from carefully counting the number of functions representable by a Transformer. The main technical ingredient is the notion of a covering number.
We will use the following definition of $\infty$-norm covering number adapted from \cite{zhang2002covering}:
\begin{definition}
[Covering number] For a given class of vector-valued functions $\mathcal{F}$, the covering number 
$\mathcal{N}_\infty(\mathcal{F}; \epsilon; \{z^{(i)}\}_{i=1}^m; \| \cdot \|)$ 
is the smallest size of a collection (a cover) $\mathcal{C} \subset \mathcal{F}$ such that $\forall f \in \mathcal{F}, \exists \hat{f} \in \mathcal{C}$ satisfying
\[\max_i \|f(z^{(i)}) - \hat{f}(z^{(i)})\| \leq \epsilon.\] 
Further, define
\begin{align*}
\mathcal{N}_\infty(\mathcal{F}, \epsilon, m, \| \cdot \|) =
\sup_{z^{(1)} \dots z^{(m)}} \mathcal{N}_\infty(\mathcal{F}; \epsilon; z^{(1)}, \dots, z^{(m)}, \| \cdot \|).
\end{align*}
\end{definition}
If $\mathcal{F}$ is real-valued (instead of vector-valued), we drop the norm from the notation. Furthermore for functions parameterized by a set of parameters $\Theta$, we exploit the notation to replace $\mathcal{F}$ by $\Theta$.

Recall that for the class of linear functions, \[\mathcal{F}_{\textrm{lin}}=\{ x \mapsto w\cdot x \ : \ w\in\RR^d, \|w\|_2 \leq B_W\},\]
we have the covering number bound \citep{zhang2002covering} of
\[\mathcal{N}_\infty(\mathcal{F}; \epsilon; \{x^{(i)}\}_{i=1}^m) \leq O\left(\frac{B_X^2B_W^2}{\epsilon^2} \cdot \log\left(\frac{B_XB_W m}{ \epsilon} \right) \right),\]
where $\|x^{(i)}\|\leq B_X$ for $i \in [m]$. Importantly, note that the covering number has a mild dependence on $m$, only logarithmic; this logarithmic dependence on $m$ will be helpful when we turn our analysis to the capacity of attention mechanisms.
\newcommand{\risk}{\mathsf{risk}}

\begin{figure*}
    \centering
    \includegraphics[width=0.95\linewidth]{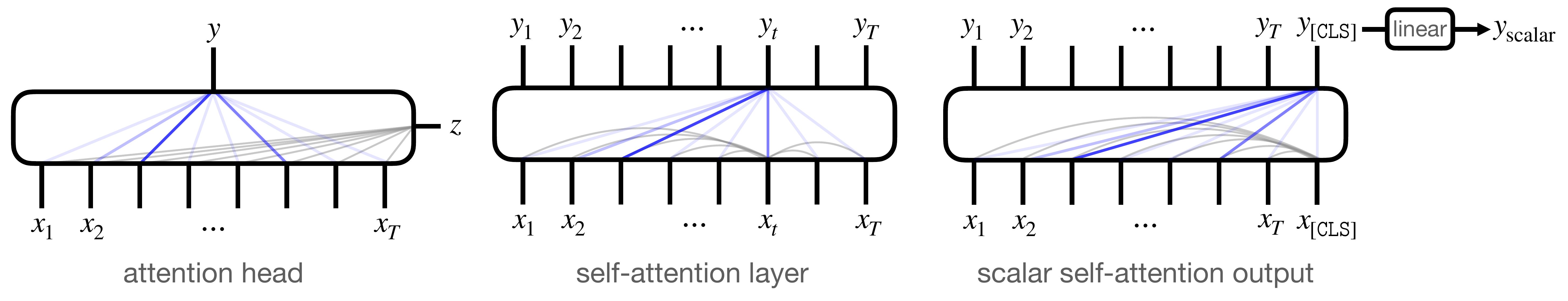}
    \vspace{-2mm}
    \caption{Diagrams of attention modules $f_{\tfHead}, f_{\tfLayer}, f_{\tfScalar}$ described in Section~\ref{subsec:prelims-attn}: alignment scores (grey edges) determine normalized attention weights (blue), which are used to mix the inputs $x_{1:T}$. \emph{Left:} Attention with a general context $z$. \emph{Center:} Self-attention layer, where both the input and the context come from $x_{1:T}$. \emph{Right:} Auxiliary $\CLS$ token to extract a single scalar from a self-attention layer, providing a real-valued function class for classification or regression tasks.}
    \label{fig:attn-diagram}
\end{figure*}
\paragraph{Generalization bounds.} This work focuses on providing log-covering number bounds, which imply uniform generalization bounds via standard arguments. The following lemma relates these quantities; we refer the reader to Appendix~\ref{sec:rad_bounds} for a formal review.
\begin{lemma}[Generalization bound via covering number; informal] \label{lem:cov-gen}
Suppose $\mathcal{F}$ is a class of bounded functions, and $\log \mathcal{N}_\infty(\mathcal{F}; \epsilon; x^{(1)}, \dots, x^{(m)}) \le C_\mathcal{F}/\eps^2$ for all $x^{(1)}, \ldots, x^{(m)} \in \mathcal{X}^m$. Then for any $\delta > 0$, with probability at least $1 - \delta$, simultaneously for all $f \in \mathcal{F}$, the generalization error $\eps_\mathrm{gen}$ satisfies
\begin{align*}
	& \eps_\mathrm{gen}(f) \le  \tilde O\left(\sqrt{\frac{C_\mathcal{F}}{m}} + \sqrt{\frac{\log(1/\delta)}{m}} \right).
\end{align*}
\end{lemma}

\section{Abstractions of (self-)attention}
\label{subsec:prelims-attn}
The precise definition of attention is less straightforward to define than for architectural components such as convolutions and residual connections.
In this section, guided by the manifestations of attention discussed in \citep{luong2015effective}, we present some notation and definitions which generalize attention mechanisms commonly seen in practice, including the Transformer. Intuitively, these definitions encompass neural network layers which induce context-dependent representation bottlenecks. Subsequently, we show how to represent the Transformer (the predominant attention-based architecture) as a special case of this formulation.

\subsection{Attention}

Intuitively, we would like to capture the notion that an output variable selects (``attends to'') a part of the input sequence on which it will depend, based on a learned function of global interactions (see Figure~\ref{fig:attn-diagram}, left). To this end, we define an \emph{attention head} as a function which maps a length-$T$ input sequence (e.g. the tokens in a sentence, pixels in an image, or intermediate activations in a deep Transformer network) and an additional context $z \in \gZ$ to an output $y \in \gY$. In this work, we will exclusively consider $\gX, \gY, \gZ$ to be $\RR^d$. An attention head uses $z$ to select the input coordinates in $X$ to which the output $y$ will ``attend'', formalized below:

\newcommand{\score}{\mathsf{Score}}
\newcommand{\normal}{\mathsf{Norm}}

\begin{definition}[Attention head]
An attention head is a function $f : \gX \rightarrow \gY$, specified by an alignment score function $\score : \gX \times \gZ \rightarrow \RR$ parameterized by $\theta_s \in \Theta_s$, normalization function $\normal : \RR^T \rightarrow \Delta^{T-1}$, and position-wise maps $\phi_\mathrm{in} : \gX \rightarrow \gV, \phi_\mathrm{out} : \gV \rightarrow \gY$ parameterized by $\theta_\mathrm{in} \in \Theta_\mathrm{in}$ and $\theta_\mathrm{out} \in \Theta_\mathrm{out}$. The output of an attention head on input $X \in \gX^T, z \in \gZ$ is\looseness=-1
\begin{align*}
y
&= \phi_\mathrm{out}\Big(\sum_{t=1}^T  \Big[\normal\Big(\score(x_1,z; \theta_s), \ldots,
\score(x_T,z; \theta_s)\Big)\Big]_t \phi_\mathrm{in}(x_t; \theta_\mathrm{in});\ \theta_\mathrm{out}\Big) \\
&= \phi_\mathrm{out}\Big(\phi_\mathrm{in}(X;\theta_\mathrm{in})^\top\normal\Big(\score(x_1,z; \theta_s), \ldots, \score(x_T,z; \theta_s)\Big) ;\ \theta_\mathrm{out}\Big) 
\end{align*}
where $\phi_{\mathrm{in}}(X; \theta) = [\phi_{\mathrm{in}}(x_1; \theta) \ldots \phi_{\mathrm{in}}(x_T; \theta)]^\top$ 
denotes the row-wise application of $\phi_{\mathrm{in}}$.
\end{definition}
The above definition corresponds to the leftmost diagram in Figure~\ref{fig:attn-diagram}. Here, $\gV$ is a vector space of input representations ``mixed'' by the normalized alignment scores; in this work, we will set $\gV = \RR^k$. A function class of attention heads is induced by specifying parameter classes for $\{ \Theta_s, \Theta_\mathrm{in}, \Theta_\mathrm{out} \}$.

\subsection{Self-attention and Transformers}

A \emph{self-attention head} is a special case of an attention head, in which the context $z$ comes from one of the inputs $x_t$ themselves: pairwise interactions among the elements in $X$ are used to select the elements of $X$ on which $f$ depends. In this case, we use ``input'' and ``context'' interchangeably to refer to $X$.
For example, a self-attention head which uses $z := x_t$ is defined by
\[y = \phi_\mathrm{out}\left( \phi_\mathrm{in}(X; \theta_\mathrm{in})^\top \normal(\score(X, x_t; \theta_s));\theta_\mathrm{out}\right).\]

We now define the Transformer self-attention architecture as a special case of the above. Since a Transformer layer has shared parameters between multiple output heads, we will define all $T$ outputs of the layer at once.

\begin{definition}[Transformer layer]
\label{def:transformer}
A Transformer layer is a collection of $T$ attention heads (whose outputs are $y_1, \ldots, y_T$) with the following shared parameters:
\begin{itemize}[leftmargin=0.5cm]
    \item The context for head $\tau$ is $x_\tau$, and the alignment score function is quadratic:
    \[\score(x, x_\tau; \{W_Q, W_K\}) := x_\tau^\top W_Q W_K^\top x,
    \quad W_Q, W_K \in \RR^{d \times k}.\]
    \item $\phi_\mathrm{in}$ is linear:
    \[ \phi_\mathrm{in}(x; W_V) := W_V^\top x, \quad W_V \in \RR^{d \times k}. \]
    \item $\phi_\mathrm{out}$ is a linear function, composed with an $L_\sigma$-Lipschitz activation function $\sigma:\RR \rightarrow \RR$ such that $\sigma(0) = 0$:
    \[\phi_\mathrm{out}(x; W_C) := W_C^\top \sigma(x), \quad W_C \in \RR^{k \times d}.\]
    \item The normalization function is the $T$-dimensional softmax: \[\normal(x) := \mathsf{softmax}\left( x\right) = \frac{\exp\left( x\right)}{\mathbf{1}^\top \exp\left( x\right)}.\]
\end{itemize}

Defining $Y := [y_1 y_2 \dots y_T]^\top \in \mathbb{R}^{T \times d}$ and $\bra{ \mathsf{RowSoftmax}(M)}_{t,:} := \mathsf{softmax}(M_{t,:})$, we have
\[
Y = \sigma\left( \mathsf{RowSoftmax}\left(X W_Q ( X W_K)^\top \right) X W_V \right)W_C.
\]
\end{definition}

Functions from the above class of Transformer layers map $\RR^{T \times d}$ to itself, and can thus be iteratively composed. We discuss remaining discrepancies between Definition~\ref{def:transformer} and real Transformers (positional embeddings, position-wise feedforward networks, layer normalization, parallel heads, residual connections) in Section~\ref{subsec:network-capacity} and the appendix.

\paragraph{Extracting scalar outputs from a Transformer.} Finally, we establish notation for a canonical way to extract a scalar prediction from the final layer of a Transformer. For a context of size $T$, a Transformer layer with $T+1$ inputs is constructed, with a special index \CLS.\footnote{\CLS stands for ``class'', as in ``treat the output at this position as the classifier's prediction''.} The input at this position is a vector $x_{\CLS} \in \RR^d$ (which can be fixed or trainable); the output is a linear function $w^\top y_{\CLS}$, for a trainable parameter $w \in \RR^d$. This defines a class of functions mapping $\RR^{T \times d} \rightarrow \RR$, parameterized by a Transformer layer's parameters and $w$, which we call the class of \emph{scalar-output Transformers}. This is the setup used by the classification modules in BERT \citep{devlin2018bert} and all of its derivatives.

%% file: 3-theory.tex
\section{Capacity bounds for attention modules}
\label{sec:capacity}

In this section, we present covering number-based capacity bounds for generic attention heads and Transformers, along with overviews of their proofs. Section~\ref{subsec:single-head-capacity} bounds the capacity of a general attention head. Section \ref{sec:tf-head} instantiates this bound for the case of a single Transformer self-attention head. Section~\ref{subsec:network-capacity} generalizes this bound for full depth-$L$ Transformer networks. The sample complexity guarantees for Transformers scale only \textit{logarithmically} in the context length $T$, providing rigorous grounding for the intuition that the architecture's inductive bias selects sparse functions of the context.

\paragraph{Note:} Throughout this section, assume that $\|x_t\|_2 \le B_X$ for all $t \in [T]$ (i.e. $\|X\|_{2,\infty} \leq B_X$). Note that this allows for the Frobenius norm $\|X\|_F$ to scale with $\sqrt{T}$. The key challenge throughout our analysis is to avoid incurring factors of norms which take a sum over the $t$ dimension, by constructing covers appropriately.

\newcommand{\Head}{\mathsf{head}}
\newcommand{\Scalar}{\mathsf{scalar}}

\subsection{Capacity of a general attention head}\label{subsec:single-head-capacity}
Recall that the attention head architecture can be represented as a function $f_\Head: \RR^{T \times d} \times \RR^d \rightarrow \RR^d$ parameterized by $\theta_s, \theta_\mathrm{in}, \theta_\mathrm{out}$ as 
\begin{align*}
f_\Head(X, z; \theta_s, \theta_\mathrm{in}, \theta_\mathrm{out}) = \phi_\mathrm{out}\left(\phi_\mathrm{in}(X; \theta_\mathrm{in})^\top \normal(\score(X, z; \theta_s));\theta_\mathrm{out}\right).
\end{align*}
Denote the corresponding function class by
\begin{align*}
\mathcal{F}_\Head:= \{(X, z) \mapsto f_\Head(X, z; \theta_s, \theta_\mathrm{in}, \theta_\mathrm{out}) :
\theta_s \in \Theta_s, \theta_\mathrm{in} \in \Theta_\mathrm{in}, \theta_\mathrm{out} \in \Theta_\mathrm{out} \}
\end{align*}

To convert the vector-valued function class to a scalar output function class, we define $\mathcal{F}_\Scalar:= \{(X, z) \mapsto w^\top f(X, z) : f \in \mathcal{F}_\Head, w \in \RR^d, \norm{w} \le B_w\}$.

For simplicity of presentation, we will focus only on the attention head, and assume that $\phi_\mathrm{out}$ and $w$ are fixed. We handle the general case of trainable downstream layers in the analysis of multi-layer Transformers in Appendix \ref{subsec:deeptf}.

\begin{assumption}\label{ass:main}
We make the following assumptions:
\begin{enumerate}[noitemsep,leftmargin=*]
    \vspace{-2ex}\item $\phi_\mathrm{out}$ is $L_\mathrm{out}$-Lipschitz in the $\ell_2$-norm, that is,
    \[\forall a, b \in \R^k: \|\phi_\mathrm{out}(a) - \phi_\mathrm{out}(b)\| \le L_\mathrm{out} \|a - b\|.\]
    \item $\phi_\mathrm{in}$ is $B_\mathrm{in}$-bounded in $\ell_2$-norm, that is,
    \[
    \forall a \in \RR^d, \theta_\mathrm{in} \in \Theta_\mathrm{in}: \| \phi_\mathrm{in}(a;\theta_\mathrm{in})\| \le B_\mathrm{in} \|a\|.
    \]
    \item $\normal$ is continuously differentiable and its Jacobian satisfies \[\forall \theta\in \RR^{T}, \norm{J~\normal(\theta)}_{1,1} \le C_\normal.\]
\end{enumerate}
\end{assumption}
Note that $\softmax$ (the most commonly used $\normal$ function) satisfies the Jacobian assumption with $C_\softmax = 2$ (see Corollary \ref{lem:softmax}).

We prove the following bound on the covering number of $\mathcal{F}_\Head$ for $m$ samples,
\begin{theorem}[Attention head capacity]\label{lem:covering_head}  Under Assumptions \ref{ass:main}, $\forall \alpha \in [0,1]$ the covering number of $\mathcal{F}_\Head$ satisfies
\begin{align*}
& \!\!\!\! \log \mathcal{N}_\infty \left(\mathcal{F}_\Head; \epsilon; \left\{(X^{(i)}, z^{(i)})\right\}_{i=1}^m; \| \cdot\|_2\right) \leq \\
& \log \mathcal{N}_\infty\left(\mathcal{F}_\score; \frac{\alpha\eps}{C_{\normal} L_{\mathrm{out}} B_\mathrm{in} B_X}; \{(x_t^{(i)}, z^{(i)})\}^{i \in [m]}_{t \in [T]}\right)
 + \log \mathcal{N}_\infty\left(\mathcal{F}_\mathrm{in}; \frac{(1- \alpha)\eps}{L_{\mathrm{out}}}; \{x_t^{(i)}\}^{i \in [m]}_{t \in [T]}; \|\cdot \|_2\right),
\end{align*}
where $\mathcal{F}_\score = \{ (x,z) \mapsto \score(x, z; \theta_s): \theta_s \in \Theta_s \}$, and $\mathcal{F}_\mathrm{in} = \{ x \mapsto \phi_\mathrm{in}(x; \theta_\mathrm{in}): \theta_\mathrm{in} \in \Theta_\mathrm{in} \}$. 
\end{theorem}
Note that the bound is in terms of the $\mathcal{N}_\infty$ covering number of functions that dependent on dimensions $d$ or $k$ and not $T$. The effect of $T$ only shows up in the number of samples to cover. Crucially, for some function classes (e.g. linear functions \citep{zhang2002covering}), $\mathcal{N}_\infty$ scales only logarithmically with the number of samples. This is exactly what allows us to obtain our $\log T$ capacity bounds.

Since $w$ is fixed, an $\eps$-covering of $\mathcal{F}_\Head$ directly gives us an $\eps B_w$-covering for $\mathcal{F}_\Scalar$, implying
\begin{align*}
&\log \mathcal{N}_\infty\left(\mathcal{F}_\Scalar; \epsilon; \left\{(X^{(i)}, z^{(i)})\right\}_{i=1}^m\right)
 \le  \log \mathcal{N}_\infty\left(\mathcal{F}_\Head; \epsilon/B_w; \left\{(X^{(i)}, z^{(i)})\right\}_{i=1}^m, \| \cdot \|_2\right).
\end{align*}
\paragraph{Proof overview.}
In order to prove the bound, we first show a Lipschitzness property of $f_\Head$. This property allows us to construct a cover by using covers for $\mathcal{F}_\score$ and $\mathcal{F}_\mathrm{in}$.
\begin{lemma}[$\ell_\infty$-Lipschitzness of $f_\Head$] \label{lem:lipschitz_general}
For any $\theta_s, \hat{\theta}_s \in \Theta_s, \theta_\mathrm{in}, \hat{\theta}_\mathrm{in} \in \Theta_\mathrm{in}$; for all $X \in \mathbb{R}^{T \times d}$, such that $\norm{X^\top}_{2, \infty} \le B_X$,
\begin{align*}
& \!\!\!\! \norm{f_\Head(X, z; \theta_s, \theta_\mathrm{in}, w) - f_\Head(X, z; \hat{\theta}_s, \hat{\theta}_\mathrm{in}, w)} \leq \\
&C_{\normal} L_{\mathrm{out}} B_\mathrm{in} B_X \norm{\score(X, z; \theta_s) - \score(X, z; \hat{\theta}_s)}_\infty
+L_{\mathrm{out}} \norm{\phi_\mathrm{in}(X; \theta_\mathrm{in}) - \phi_\mathrm{in}(X; \hat{\theta}_\mathrm{in})}_{2, \infty}.
\end{align*}
\end{lemma}
The most crucial aspect of this proof is to avoid a spurious $T$ dependence when accounting for the attention mechanism. The key observation here is that the attention part of the network is computed using $\normal$, whose Jacobian norm is bounded. This allows us to use the mean-value theorem to move to the maximum ($\ell_\infty$) error over $T$ tokens instead of sum ($\ell_1$), which could potentially incur a $T$ factor. Furthermore, this allows us to combine all samples and tokens and construct an $\ell_\infty$-cover directly for $mT$ samples.

\subsection{Capacity of a Transformer head} \label{sec:tf-head}
Let us now look at the case of a Transformer self-attention head and instantiate the covering bound. For ease of presentation and to focus on the self-attention part, we collapse $W_QW_K^\top$ to a single matrix, set $k=d$ and remove the linear layer $W_C$\footnote{See Appendix \ref{subsec:deeptf} for an analysis of general deep Transformer models.}. Then the Transformer self-attention head (for any fixed $\tau \in [T]$) can be described as
\begin{align*}
&f_\tfHead(X; W_V, W_{QK}) :=
\sigma\left(W_V^\top X^\top \mathsf{softmax}\left(XW_{QK}^\top x_\tau\right)\right)
\end{align*}
which is obtained from the general formulation by setting the context to be $x_\tau$, $\score(X,x_\tau;W_{QK}) = XW_{QK}^\top x_\tau $, $\normal = \softmax$ and $\phi_\mathrm{out} = \sigma$.

Because the number of parameters in a Transformer self-attention head is $O(d^2)$, with no dependence on $T$, one might presume by simple parameter counting that the capacity of the class of these heads does not grow as the context length $T$ grows. But capacity is not solely governed by the number of parameters---for example, the class $\{ x \in \R \mapsto \mathrm{sign}(\sin(\alpha x))\}_{\alpha \in \R}$ has a single parameter but infinite VC-dimension. One might still hope to prove, for the special case of Transformer heads, a $T$-independent upper bound on the VC-dimension (or rather, its analog for real-valued functions, the \emph{pseudo-dimension}). We observe that, in fact, the pseudo-dimension of this class does grow with $T$.

\begin{proposition}\label{prop:shattering}
	When the embedding dimension is $d=3$, the class \[ \mathcal{F}_{\tfHeadUnbounded} := \{X \mapsto f_\tfHead(X; W_V, W_{QK}) : W_V, W_{QK} \in \RR^{d \times d}\}\]
	of Transformer self-attention heads with unbounded norm has pseudo-dimension $\geq \lfloor\log T\rfloor$.
\end{proposition}

The proofs for this subsection can be found in Appendix \ref{sec:proofs}.

Let us now define the function class of self-attention heads with bounded weight norms:
\begin{align*}
&\mathcal{F}_\tfHead:= \{X \mapsto f_\tfHead(X; W_V, W_{QK}) :
\quad \|W_V\|_{2,1} \le B_V^{2,1}, \|W_{V}\| \le B_V, \|W_{QK}^\top\|_{2,1} \le B_{QK}^{2,1} \}.
\end{align*}
Since $W_V, W_{QK}$ have dimensions dependent on $d$ and $k$, bounding their norms does not hide a $T$ dependence. As before, to convert this vector-valued function class to a scalar output function class, we define \begin{align*}&\mathcal{F}_\tfScalar:=
\quad \{X \mapsto w^\top f(X) : f \in \mathcal{F}_\tfHead, w \in \RR^d, \norm{w} \le B_w\}.
\end{align*}


We obtain the following bound on the covering number of $\mathcal{F}_\tfHead$ as a corollary of Theorem~\ref{lem:covering_head}:
\begin{corollary}\label{lem:covering} For any $\epsilon>0$ and $X^{(1)}, \ldots, X^{(m)} \in \mathbb{R}^{T \times d}$ such that $\norm{{X^{(i)}}^\top}_{2, \infty} \le B_X$ for all $i \in [m]$, the covering number of $\mathcal{F}_\tfHead$ satisfies
\begin{align*}
\log \mathcal{N}_\infty(\mathcal{F}_\tfHead; \epsilon; X^{(1)}, \dots, X^{(m)}, \| \cdot \|_2)
\lesssim (L_\sigma B_X)^2 \cdot \frac{\left( (B_V^{2,1} )^\frac{2}{3} + (B_{QK}^{2,1} B_V)^\frac{2}{3} \right)^3}{\eps^2}\cdot\log(mT)
\end{align*}
Here $\lesssim$ hides logarithmic dependencies on quantities besides $m$ and $T$.
\end{corollary}

\paragraph{Proof overview.}
The above result follows from bounding the covering numbers of \[\mathcal{F}_{QK}:= \{z \mapsto x_\tau^\top W_{QK}z: \|W_{QK}^\top\|_{2,1} \le B_{QK}^{2,1}\}, \text{ and}\]
\[\mathcal{F}_{V}:= \{z \rightarrow W_V^\top z: \|W_{V}\|_{2,1} \le B_V^{2,1}, \|W_{V}\| \leq B_V\}.\]

Note that $|x_\tau^\top W_{QK} z - x_\tau^\top W_{QK} \hat{z}|
\leq \|W_{QK} z - W_{QK} \hat{z}\|$ since $\|x_\tau\| \leq 1$, so the covering number of $\mathcal{F}_{QK}$ is at most the covering number of the class of functions of the form $z \mapsto W_{QK} z$. Therefore, a covering number bound for the vector-valued linear function class suffices to handle both covering numbers:
\begin{lemma}\label{lem:2-inf-cover}
Let $\mathcal{W}: \{W \in \RR^{d_1 \times d_2}: \|W^\top\|_{2,1} \le B_W\}$, and consider the function class $\mathcal{F}: \{ x \mapsto Wx: W \in \mathcal{W} \}$. For any $\eps > 0$ and $x^{(1)}, \dots, x^{(N)} \in \R^{d_1}$ satisfying $\forall i \in [N], \norm{x^{(i)}} \leq B_X$,
\begin{align*}
\log \mathcal{N}_\infty(\mathcal{F}; \eps; x^{(1)}, \dots, x^{(N)}; \|\cdot\|_2)
\lesssim \frac{(B_X B_W)^2}{\eps^2} \log (d_1N).
\end{align*}
\end{lemma}

Note that this bound only depends logarithmically on the context length, as desired. The proof can be found in Appendix \ref{sec:proofs}. 

Finally, our analysis is compatible with the following additional components:

\paragraph{Positional embeddings.} In practice, the permutation-invariant symmetry of a Transformer network is broken by adding a \emph{positional embedding} matrix $P \in \RR^{T \times d}$ to the input $X$ at the first layer. In practice, the embedding matrix is often fixed (not trainable). Our results extend to this setting in a straightforward way; see Appendix \ref{sec:positional}. If these matrices are to be trained from a sufficiently large class (say, $\norm{P}_{2,\infty} \leq 1$), the dependence of the log-covering number on $T$ could become linear.\looseness=-1

\paragraph{Residual connections.} Including residual connections (e.g. redefining $f_\tfHead(X)$ as $x_t + f_\tfHead(X)$ for some index $t \in [T]$) simply increases the Lipschitz constant of each layer (w.r.t. the input) by at most $1$. As long as $B_V = \Omega(1)$, this only changes our covering number bounds by a constant factor. \looseness=-1

\paragraph{Multi-head self-attention.} In almost all applications of Transformers, multiple parallel self-attention heads are used, and their outputs aggregated, to allow for a richer representation. Our analysis directly extends to this setting; see Appendix \ref{sec:multihead} for details. When a single attention head is replaced with the sum of $H$ parallel heads, the log-covering number scales up by a factor of $\mathrm{poly}(H)$.\looseness=-1

\paragraph{Layer normalization.} State-of-the-art Transformer networks are trained with layer normalization modules \citep{ba2016layer}, which is generally understood to aid optimization. We keep a variant of layer normalization in the covering number analysis-- it proves to be useful in the analysis of full attention blocks (see Appendix~\ref{subsec:deeptf}), as it keeps the norm of the embedding of each token bounded. Removing these layers would lead to a worse dependence on the spectral norm of the matrices.

\subsection{Capacity bounds for multi-layer Transformers}
\label{subsec:network-capacity}
\newcommand{\tfBlock}{\mathsf{tf{\text-}block}}
In this section, we will extend our results for $L$-layer Transformer blocks. Denote the weights of layer $i$ by $W^{(i)} := \left\{W_Q^{(i)}, {W_K^{(i)}}, W_V^{(i)}, W_C^{(i)}\right\}$. 
Further denote the set of weights up to layer $i$ by $W^{1:i} = (W^{(1)}, \ldots, W^{i-1})$. Denote the input representation of layer $i$ by $g_\tfBlock^{(i)}(X; W^{1:i})$. We inductively define $g_\tfBlock^{(i)}: \R^{T \times d} \to \R^{T \times d}$ starting with $g_\tfBlock^{(1)}(X; W^{1:1}) = X$ (the input):
\begin{align*}
&g_\tfBlock^{(i+1)}\left(X; W^{1:i+1}\right) := \Pi_\mathsf{norm}\left(\sigma\left(\Pi_\mathsf{norm}\left(f\left(g_\tfBlock^{(i)}\left(X; W^{1:i}\right); W^{(i)}\right)\right)\right)W_C^{(i)}\right), \\
&\qquad \text{with}\ f\left(Z; \{W_Q, W_K, W_V, \cdot\}\right)  := \mathsf{RowSoftmax}\left(Z W_Q \left( Z W_K\right)^\top \right) Z W_V,
\end{align*}
where $\Pi_\mathsf{norm}$ denotes layer normalization\footnote{Layer normalization allows for the norms of the outputs of each token in each layer to remain bounded by $1$. Note that the norm of the entire input can still have a dependence on $T$. Our results would go through with a worse dependence on the spectral norms if we were to remove layer norm.} applied to each row. We use a slightly modified version of LayerNorm where instead of normalizing to norm 1, we project it to the unit ball. Let the class of depth-$L$ transformer blocks be 
\begin{align*}
\mathcal{F}^{(L)}_\tfBlock :=
&\left\{ X \rightarrow g_\tfBlock^{(L+1)}(X; W^{1:L+1}) : \forall~i \in[L], \right.\\
&\quad \norm{W_V^{(i)}}_2, \norm{W_K^{(i)}{W_Q^{(i)}}^\top}_2, \norm{W_C^{(i)}}_2 \le C_2, \\
&\quad \left.\norm{W_V^{(i)}}_{2,1}, \norm{{W_K^{(i)}}^\top{W_Q^{(i)}}}_{2,1}, \norm{W_C^{(i)}}_{2,1} \le C_{2,1} \right\}.
\end{align*}
To obtain a final scalar output, we use a linear function of the $\CLS$ output:
\[g_\tfScalar(X;W^{1:L+1}, w) = w^\top \bra{ g\left(X; W^{1:L+1}\right) }_{\CLS,:}.
\]
Let the scalar output function class be $\mathcal{F}^{(L)}_\tfScalar:=\{X \rightarrow w^\top f(X)_{\CLS} : f \in \mathcal{F}^{(L)}_\tfBlock, w \in \RR^d, \norm{w} \le B_w\}$.

\begin{theorem}[Theorem \ref{thm:deeptf_full} (simplified)] \label{thm:deeptf}
Suppose $\forall i \in [m], \left\|X^{(i)}\right\|_{2,\infty} \leq B_X$, then we have
\begin{align*}
\log \mathcal{N}_\infty&(\mathcal{F}_\tfBlock^{(L)}; \epsilon; X^{(1)}, \dots, X^{(m)})
\lesssim (C_{2}L_\sigma)^{O(L)} \cdot \frac{ B_X^2B^2_w C_{2,1}^2 }{\eps^2} \cdot \log(dmT).
\end{align*}
\end{theorem}
Note that the dependence on $d$ and $T$ is only logarithmic even for deeper networks. The dependence on $(2,1)$-norms of the weight matrices is quadratic. As long as the spectral norms of the matrices are bounded by $1$ and $\sigma$ is 1-Lipschitz (which holds for sigmoids and ReLUs), the exponential dependence on $L$ can be avoided.

\newcommand{\diag}{\mathrm{diag}}

\section{Attention approximates sparse functions}

\label{sec:sparse-rep}
The results in Section~\ref{sec:capacity} show that function classes bottlenecked by self-attention mechanisms have ``small'' statistical capacity in terms of the context size. In this section, we answer the converse question: \emph{which functions of interest are in these classes?} We show that Transformers are able to represent sparse interactions in the context with bounded weight norms, and can thus learn them sample-efficiently.

Consider the class of Boolean functions $f : \{0,1\}^T \rightarrow \RR$ which are \emph{$s$-sparse}: they only depend on $s \ll T$ of their inputs. We will construct mappings from such functions to parameters of a self-attention head $f_\tfHead$ composed with a feedforward network $f_\mlp$; note that $f_\tfHead \circ f_\mlp$ is the standard Transformer block. Intuitively, $f_\tfHead$ is constructed to ``keep'' the correct $s$-dimensional subset of inputs and ``forget'' the rest, while $f_\mlp$ ``memorizes'' the values of $f$ on these $s$ inputs, using $2^{O(s)}$ parameters.

\paragraph{Setup.} We consider the classes of Boolean functions $f : \{0,1\}^T \rightarrow \RR$ representable by bounded-norm scalar-output Transformer heads $f_\tfScalar : \RR^{T \times d} \rightarrow \RR$. To do this, we must first fix a mapping from $\{0,1\}^T$ to $\RR^{T \times d}$; we discuss several natural choices in Appendix~\ref{subsec:appendix-approx-setup}. The simplest of these uses a sum of token and positional embeddings $X(b)_{t,:} := e_{b_t} + v_t$, for a set of approximately orthogonal unit vectors $\{e_0, e_1\} \cup \{ v_1, \ldots, v_T\}$ of dimension $d = \Theta(\log T)$.
After choosing a mapping $X(b)$, the setup of the representation problem is as follows: given $f(b)$, find Transformer weights $\theta_\tfHead$ and feedforward network weights $\theta_\mlp$ such that
\begin{align*}
    &f_\tfmlp(X(b) ; \theta_\tfHead, \theta_\mlp) :=
    \quad f_\mlp \pa{ f_\tfHead( X(b) ; \theta_\tfHead) ; \theta_\mlp } \approx f(b), \forall b \in \{0,1\}^T.
\end{align*}

\paragraph{Main representational results.}
For any size-$s$ subset of indices $\gI \subseteq [T]$,
we show that Transformer blocks can represent all $\gI$-sparse Boolean functions, whose values only depend on the inputs at the coordinates in $\gI$.
We give informal statements of these approximation results below, and present the precise statements in Appendix~\ref{subsec:appendix-approx-results}.

\begin{proposition}[Sparse variable creation via Transformers; informal]
\label{prop:sparse-variable-creation}
Under any of the input mappings $X(b)$, we have the following guarantees:
\begin{itemize}[itemsep=0.08cm,leftmargin=*]
\vspace{-2ex}\item $f_\tfScalar$ can approximate a particular monotone symmetric $s$-sparse Boolean function, with norms $\norm{W_Q}_F \leq O\pa{\log(Ts)}; \norm{W_K}_F, \norm{W_V}_F, \norm{W_C}_F \leq O(s)$.
\item $f_\tfmlp$ can exactly represent symmetric $s$-sparse functions, with the same Transformer weight norms as above; the feedforward network weights satisfy $\norm{W_1}_F, \norm{W_2}_F, \norm{w}_F \leq O(\mathrm{poly}(s))$.
\item $f_\tfmlp$ can exactly represent general $s$-sparse functions, with the same Transformer weight norms as above; the feedforward network weights satisfy $\norm{W_1}_F, \norm{W_2}_F, \norm{w}_F \leq O(2^s \cdot \mathrm{poly}(s))$.
\end{itemize}
\end{proposition}

These results and the capacity bounds from Section~\ref{sec:capacity} are simultaneously meaningful in the regime of $s \ll \log T$.
An appealing interpretation for the $s = 2$ case is that a single Transformer head can learn a single logical gate (i.e. $\mathsf{AND}, \mathsf{OR}, \mathsf{NAND}, ...$) in a Boolean circuit, with $d$ and weight norms scaling as $\log T$.

\paragraph{Proof ideas.}
Each construction uses the same basic idea: select $W_Q, W_K$ so that the attention mixture weights approximate the uniform distribution over the relevant positions, then use the ReLU network to memorize all distinct values of $f$. Full proofs are given in Appendix~\ref{subsec:appendix-approx-proofs}.

\paragraph{Other realizable functions.} Since there are $\binom{T}{s}$ $s$-sparse subsets of input indices, the sample complexity of learning a sparse Boolean function must scale at least as $\Omega(s \log T)$, matching the capacity bounds in terms of the $\log T$ dependence. However, sparse functions are not the \emph{only} potentially useful functions realizable by bounded-norm Transformers. For instance, with $W_Q = W_K = 0$, so that all scores are zero, a Transformer head can take an average of $T$ embeddings $X \rightarrow \frac{1}{T} \mathbf{1}^\top X W_V$. More generally, departing from the ``orthogonal context vectors'' embedding of Boolean inputs but using the same constructions as in this section, it is straightforward to conclude that bounded-norm Transformers can compute global averages of tokens whose $XW_K$ embeddings lie in an $s$-dimensional subspaces. This is why our results do not contradict the empirical finding of \citet{clark2019does} that some attention heads in trained Transformer models attend broadly. It is also straightforward to extend some of these results beyond Boolean domains; see Section~\ref{subsec:appendix-beyond-boolean} for a sketch.

\paragraph{Bypassing theoretical limitations.} \citet{hahn2020theoretical} points out that with constant weight norms, a Transformer's ability to express global dependencies degrades with context length: as $T \rightarrow \infty$, the maximum change in output caused by altering a single input token approaches 0, and thus various interesting formal languages cannot be modeled by a Transformer in this particular limit. The constructions in this section show that this can be circumvented by allowing $d$ and the weight norms to scale as $\log(T)$.

%% file: 4-experiments.tex
\section{Experiments}

Sections~\ref{sec:capacity} and \ref{sec:sparse-rep} show theoretically that Transformers can learn sparse Boolean functions, with sparse regression-like sample complexity (in terms of the $\log T$ dependence). In this section, we present an empirical study which probes the end-to-end sample efficiency of Transformer architectures with standard training and architecture hyperparameters, and how it scales with the context length $T$.

\begin{figure}
    \centering
    \includegraphics[width=0.4\linewidth]{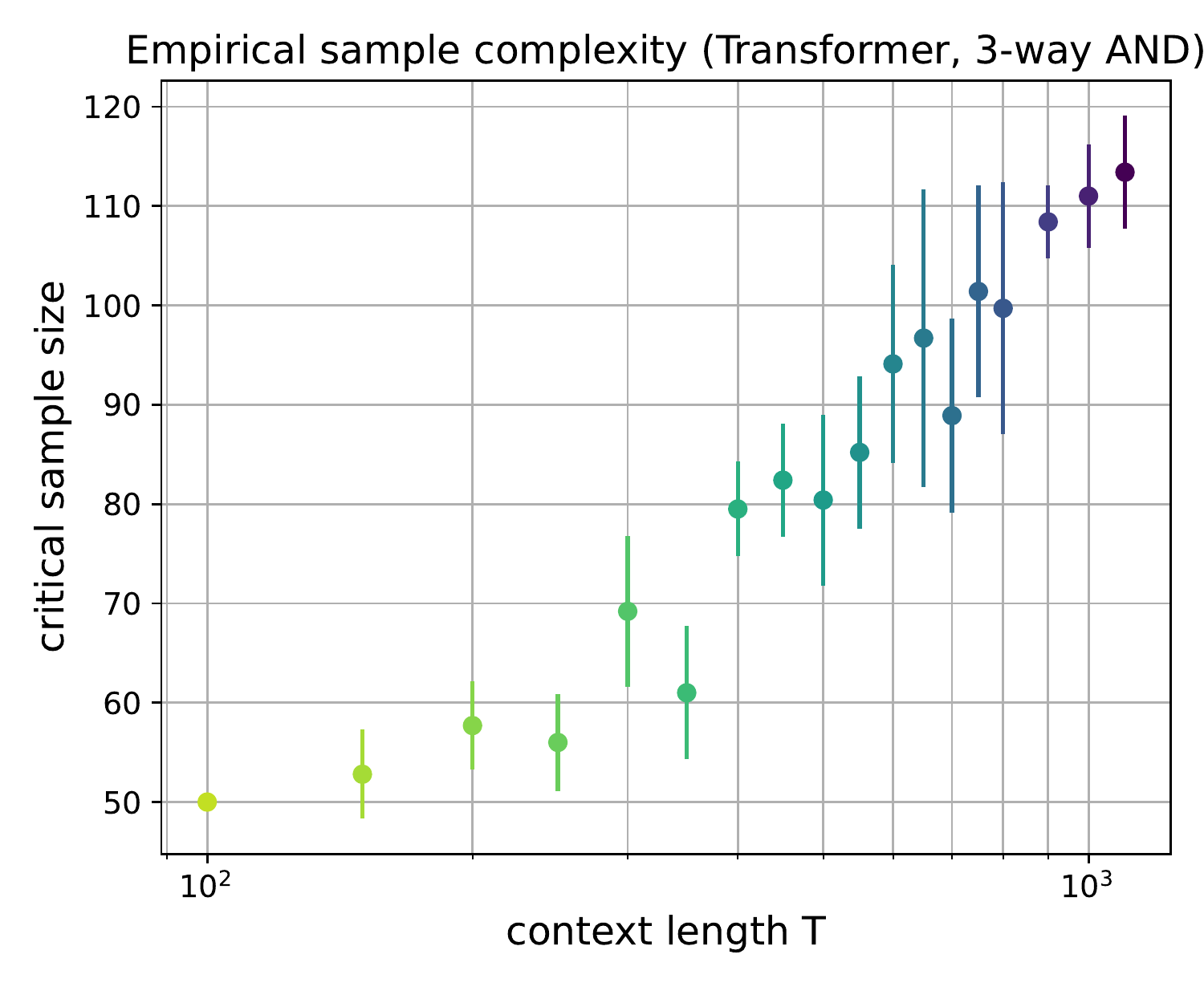}
    \label{fig:sparse-and-scaling-laws}
    \caption{Main experimental finding: the sample complexity of learning a $3$-sparse $\mathsf{AND}$ function of $T$ input bits with Transformers. For each $T$, we measure the smallest sample size $m$ necessary to reach $100\%$ validation accuracy on $\geq 80\%$ of random trials. We find that this threshold scales logarithmically with $T$.}
\end{figure}

\paragraph{Setup.} We introduce a synthetic benchmark to support our analysis, in which we measure the statistical limit for learning sparse Boolean functions with Transformers.
We choose a distribution $\gD$ on $\{0,1\}^T$, and a family of distinct functions $\{ f_i : \{0,1\}^T \rightarrow \{0,1\} \}_{i=\in[N]}$, where $N$ grows with $T$. Then, we choose an $i^* \in [N]$ uniformly at random, and train a Transformer binary classifier on $m$ samples from $\gD$, with labels given by $f_{i^*}$, evaluating generalization error via holdout samples. Then, for any learner to reach $100\%$ accuracy on this sample, $m \geq \Omega(\log N)$ samples are required (one sample reveals at most one bit of information about $i^*$). We can then measure the empirical scaling of the sufficient sample size $m$ to solve this problem, in terms of $N$ (and thus $T$).

\paragraph{Learning sparse conjunctions.} Concretely, we can choose $f_i$ be the set of all $\binom{T}{s}$ conjunctions of $s$ inputs (e.g. $y = x_2 \wedge x_3 \wedge x_{10}$), fixing the input distribution $\gD$ to be i.i.d. Bernoulli (we choose the bias to balance the labels). The model must learn which subset of $s$ features are relevant, out of $\binom{T}{s}$ possibilities; this requires at least $m \geq \Omega(s \log T)$ samples. The theoretical analysis predicts that the sample complexity of learning any function realizable by a bounded-norm Transformer should asymptotically have the same $\log T$ scaling. We choose a fixed sparsity parameter $s=3$, and measure how the empirical sample complexity (i.e. the smallest sample size $m(T)$ at which model training succeeds with non-negligible probability) scales with $T$.

\begin{figure}
    \centering
    \includegraphics[width=0.4\linewidth]{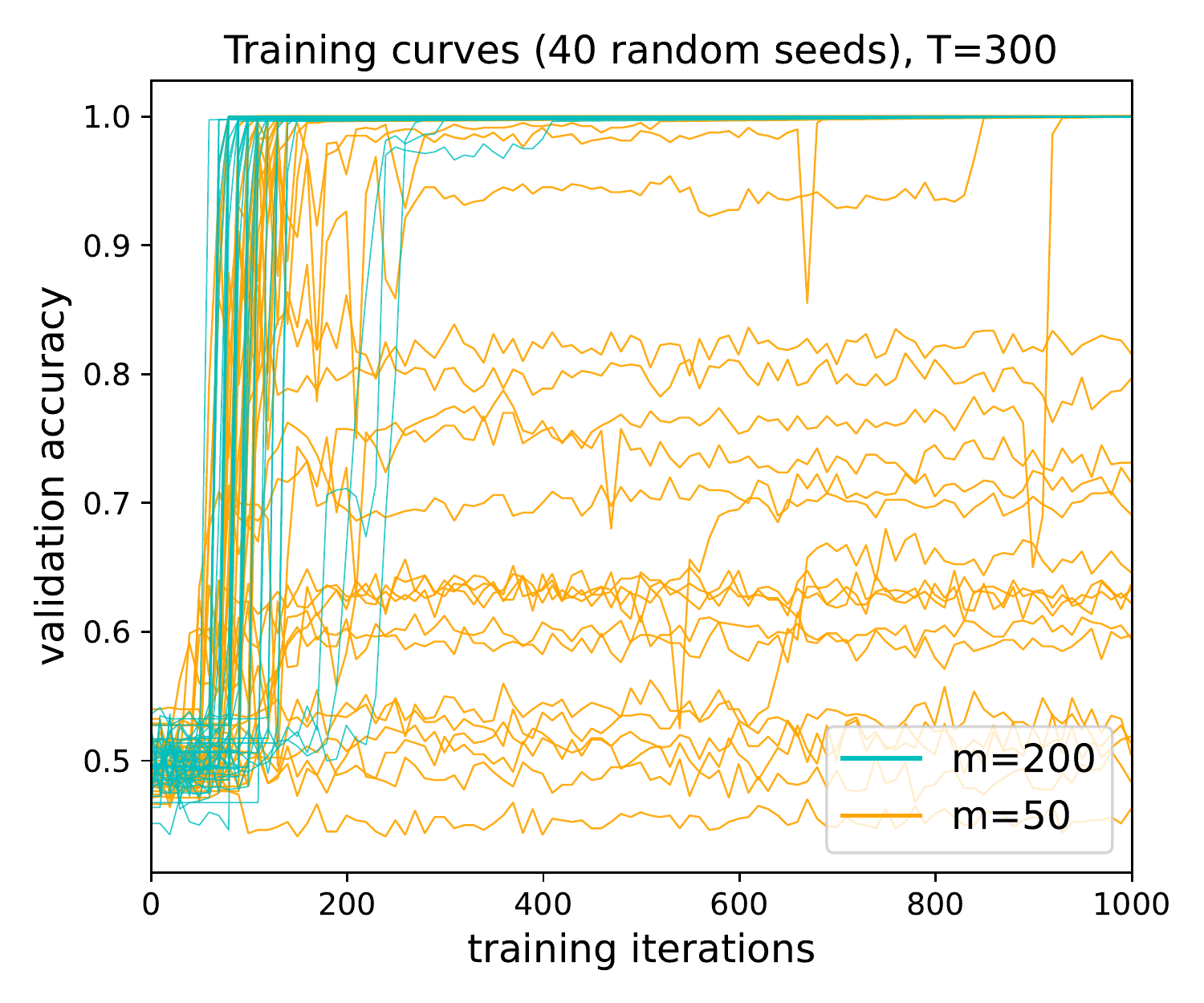}
    \hspace{-4mm}
    \includegraphics[width=0.4\linewidth]{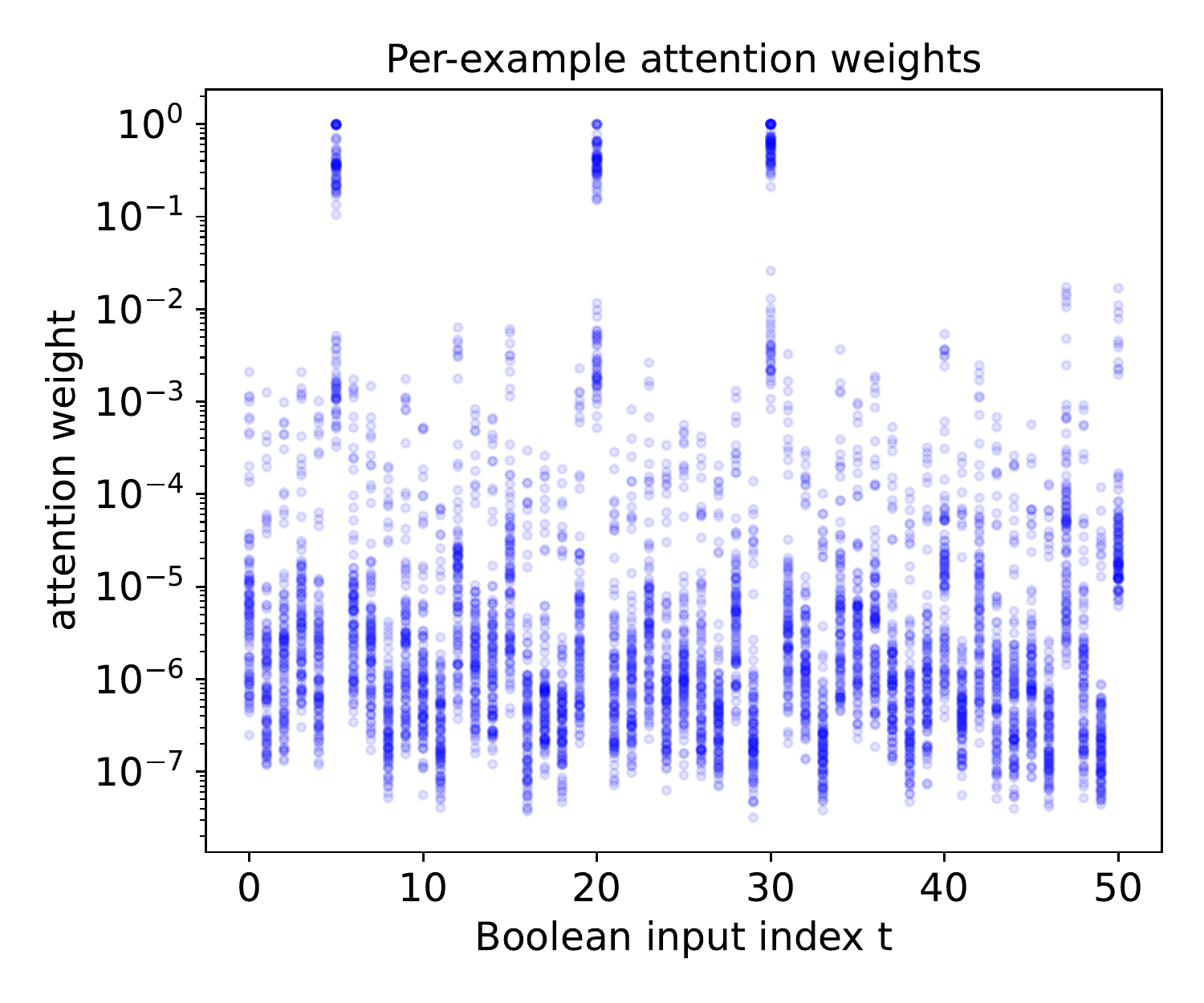}
    \caption{Additional visualizations for the sparse function learning experiments. \emph{Left:} Examples of validation accuracy curves on the same problem instance ($T = 300$), with sample sizes above ($m = 200$) and below ($m = 50$) the threshold ($\approx 70$ from Figure~\ref{fig:sparse-and-scaling-laws}). Training accuracy goes to $100\%$ in both cases, but the Transformer overfits (orange curves) when $m$ is too small. \emph{Right:} Per-example attention weights for a successfully trained model ($T = 50$, $m = 300$, $\gI = \{5, 20, 30\}$). The input-dependent attention weights approximately zero out the irrelevant bits. }
    \label{fig:sparse-and-extra}
\end{figure}

\paragraph{Results.} With architecture and training hyperparameters typical of real Transformer setups (except the number of layers, which we set to $1$), we indeed observe that the empirical sample complexity appears to scale as $\log T$; see Figure~\ref{fig:sparse-and-scaling-laws}. Despite the exponentially large support of input bit strings $x$ and large total parameter count ($\sim 10^5$), the attention weights vanish on the $T-s$ irrelevant coordinates, and the model converges to sparse solutions; this is visualized in Figure~\ref{fig:sparse-and-extra} \emph{(right)}. Details are provided in Appendix~\ref{subsec:appendix-sparse-and}; in particular, model training near the statistical threshold is extremely unstable, and extensive variance reduction (best of $5$ random restarts; $40$ replicates; a total of $\sim 10^4$ training runs across each $T,m$) was necessary to produce these scaling plots.

\paragraph{Beyond our analysis: sparse parities.} When choosing the family of sparse functions $\{f_i\}$, we can replace the $\mathsf{AND}$ operation with $\mathsf{XOR}$: the label is the parity of a randomly chosen subset of i.i.d. uniform input bits. In this setting, unlike the $\mathsf{AND}$ case, there is a computational-statistical gap: $\Theta(s \log T)$ samples suffice to identify, but the fastest known algorithms for learning parities with noise require $T^{\Omega(s)}$ time. In the statistical query model, $\Omega(T^s)$ iterations of noisy batch gradient descent are necessary \citep{kearns1998efficient}.
Figure~\ref{fig:sparse-parity} (with details in Appendix~\ref{subsec:appendix-sparse-parity}) shows that when trained with i.i.d. samples, Transformer models can learn sparse parities.
This raises an intriguing question, which is the \emph{computational} analogue of the current work's \emph{statistical} line of inquiry: how does local search (i.e. gradient-based training) succeed at finding solutions that correspond to sparse discrete functions? The present work merely shows that these solutions exist; we intend to address the computational mystery in future work.

\begin{figure}
    \centering
    \includegraphics[width=0.99\linewidth]{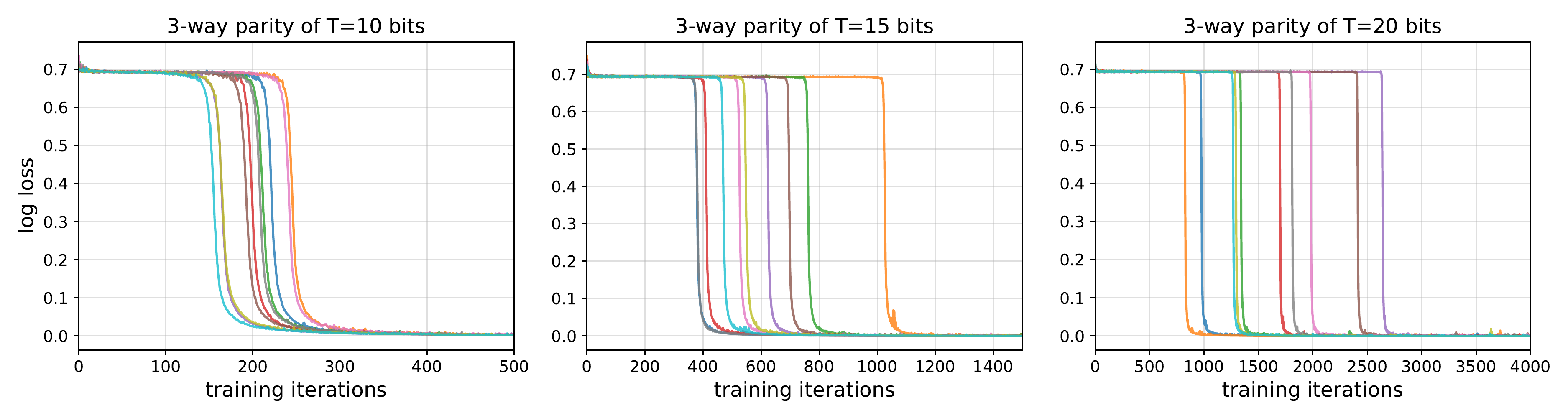}
    \vspace{-3mm}
    \caption{A curious empirical finding: Transformers can learn sparse parities. Loss curves (across 10 random seeds for initialization and SGD samples) are shown for this setup with $s=3, T \in \{10,15\}$, exhibiting phase transitions from random guessing to $100\%$ accuracy. See Appendix~\ref{subsec:appendix-sparse-parity} for details.}
    \label{fig:sparse-parity}
\end{figure}

%% file: 5-conclusion.tex
\section{Conclusion and future work}

This work establishes a statistical analysis of attention and self-attention modules in neural networks. In particular, we identify an inductive bias we call \emph{sparse variable creation}, consisting of (1) covering number-based capacity bounds which scale as $\log T$, and (2) constructions which show that self-attention models with small weight norms can represent sparse functions.
This analysis is supported by an empirical study on learning sparse Boolean functions with Transformers. We hope that these rigorous connections between attention and sparsity, as well as the proposed experimental protocols, will inform the practice of training and regularizing these models, and the design of future attention-based architectures.

We believe that it is possible to refine the covering number bounds (where we have only sought to obtain optimal dependences on $T$) as well as the representation results (where we have not used the structure of the MLP, beyond its capacity for exhaustive memorization). Significant challenges (which are not specific to attention) remain in closing the theory-practice gap: precisely understanding the role of depth, as well as the trajectory of the optimization algorithm.

An exciting line of empirical work has made progress on understanding and interpreting state-of-the-art Transformer language models by examining the activations of their attention mechanisms \citep{clark2019does,tenney2019bert,rogers2020primer}. In some cases, these works have found instances in which Transformers seem to have learned features that are reminiscent of (sparse) hand-crafted features used in natural language processing. Reconciling our theoretical foundations work with this area of \emph{BERTology} is an avenue for future synthesis.

\section*{Acknowledgements}

Sham Kakade acknowledges funding from the Office of Naval Research under award N00014-22-1-2377 and the National Science Foundation Grant under award \#CCF-1703574. We thank Nati Srebro for his questions regarding the removal of the dimension factor in an earlier version of this manuscript.

%% file: A-capacity-proofs.tex

\newpage
\section{Proofs of capacity bounds} \label{sec:proofs}

In this section we present the full proofs (including the omitted proofs) of our capacity bounds. We also cover relevant background and useful technical lemmas.
\subsection{Rademacher complexity and generalization bounds}\label{sec:rad_bounds}
Here we briefly review Rademacher complexity and its relationship to covering numbers and generalization bounds. We refer the reader to \cite{BM:2002} for a more detailed exposition.

\begin{definition}[Empirical Rademacher complexity]
For a given class of functions $\mathcal{F} = \{f : \gX \rightarrow \RR\}$ and $\{ z^{(i)} \in \gX \}_{i=1}^m$, the empirical Rademacher complexity $\hat{\mathcal{R}}(\mathcal{F}; z^{(1)}, \dots, z^{(m)})$ is defined as
\[
\hat{\mathcal{R}}(\mathcal{F}; z^{(1)}, \dots, z^{(m)}) = \frac{1}{m} \mathbb{E}_{\epsilon}\left[\sup_{f \in \mathcal{F}} \sum_{i=1}^m \epsilon_i f(z^{(i)})\right],
\]
where $\epsilon$ is a vector of $m$ i.i.d. Rademacher random variables ($\Pr[\epsilon_i = 1] = \Pr[\epsilon_i=-1] =1/2$).
\end{definition}

In order to relate the Rademacher complexity and $\ell_\infty$-covering numbers, we use a modified version of Dudley's metric entropy. 
\begin{lemma}[\citet{dudley1967sizes}; modified]\label{lem:Dudley}
Consider a real-valued function class $\mathcal{F}$ such that $|f| \leq A$ for all $f \in \mathcal{F}$. Then
\[
\hat{\mathcal{R}}(\mathcal{F}; z^{(1)}, \dots, z^{(m)})
\le c \cdot \inf_{\delta \geq 0} \left(\delta +
\int_\delta^A \sqrt{\frac{\log \mathcal{N}_\infty(\mathcal{F}; \epsilon; z^{(1)}, \dots, z^{(m)})}{m}} \, d\epsilon\right)
\]
for some constant $c >0$.
\end{lemma}
\begin{proof}[Proof sketch]
The original statement is for 2-norm covering number, but the $\infty$-norm case reduces to the 2-norm case because $N_2(\cdot) \leq N_\infty(\cdot)$. The original statement also fixes $\delta=0$ rather than taking an infimum. Also, the standard statement has the integral going from 0 to $\infty$, but these are easily replaced with $\delta$ and $A$.
\end{proof}

For our paper, we will instantiate the above lemma for log covering numbers scaling as $1/\eps^2$.
\begin{corollary}[Rademacher complexity via covering number]\label{lem:cov-rad}
Consider a real-valued function class $\mathcal{F}$ such that $|f| \leq A$ for all $f \in \mathcal{F}$. Suppose $\log \mathcal{N}_\infty(\mathcal{F}; \epsilon; z^{(1)}, \dots, z^{(m)}) \le C_\mathcal{F}/\eps^2$, then
\[
\hat{\mathcal{R}}(\mathcal{F}; z^{(1)}, \dots, z^{(m)})
\le c \cdot \sqrt{\frac{C_\mathcal{F}}{m}} \cdot \left(1 + \log \left(A \sqrt{m/C_\mathcal{F}}\right)\right)
\]
for some constant $c > 0$.
\end{corollary}
\begin{proof}
Using Lemma \ref{lem:Dudley}, we have for some constant $c > 0$,
\begin{align*}
   \hat{\mathcal{R}}(\mathcal{F}; z^{(1)}, \dots, z^{(m)})& \le c  \inf_{\delta \geq 0} \left(\delta +
\int_\delta^A \sqrt{\frac{\log \mathcal{N}_\infty(\mathcal{F}; \epsilon; z^{(1)}, \dots, z^{(m)})}{m}} \, d\epsilon\right)\\
&\le c\inf_{\delta \geq 0} \left(\delta +\int_\delta^A \sqrt{\frac{C_\mathcal{F}}{\eps^2 m}} \, d\epsilon\right)\\
&= c\inf_{\delta \geq 0} \left(\delta + \sqrt{\frac{C_\mathcal{F}}{m}}\int_\delta^A \frac{1}{\eps} \, d\epsilon\right)\\
&= c\inf_{\delta \geq 0} \left(\delta + \sqrt{\frac{C_\mathcal{F}}{m}}\log(A/\delta)\right)\\
&= c\sqrt{\frac{C_\mathcal{F}}{m}} \left(1 + \log \left(A \sqrt{m/C_\mathcal{F}}\right)\right).
\end{align*}
\end{proof}

We can now obtain a generalization guarantee from the Rademacher complexity of a function class:
\begin{theorem}[\citet{BM:2002}]\label{thm:gen}
Let $\mathcal{D}$ be a distribution over $\mathcal{X} \times \RR$ and let $\ell : \RR
	\times \RR$ be a $b$-bounded loss function that is $L$-Lipschitz in its first argument.  For a given function class $\mathcal{F}$ and $f \in \mathcal{F}$, let $\risk(f; \mathcal{D}) := \mathbb{E}_{(x, y)
	\sim \mathcal{D}}[\ell(f(x), y)]$ and $\hat{\risk}\left(f; (z^{(i)}, y^{(i)})_{i=1}^m\right) := \frac{1}{m} \sum_{i = 1}^m\ell(f(z^{(i)}),y^{(i)})$. Then for any $\delta > 0$, with probability at least $1 - \delta$, simultaneously for all $f \in
	\mathcal{F}$,
	\[
		\left|\risk(f; \mathcal{D}) - \hat{\risk}\left(f; (z^{(i)}, y^{(i)})_{i=1}^m\right)\right| \leq 4 L \, \hat{\mathcal{R}}\left(\mathcal{F}; z^{(1)}, \dots, z^{(m)}\right)
		+ 2 b \sqrt{\frac{\log (1/\delta)}{2m}}.
	\]
\end{theorem}

Combining the above, we get:
\begin{lemma}[Lemma \ref{lem:cov-gen} (restated)] 
Consider a function class $\mathcal{F}$ such that $|f| \leq A$ for all $f \in \mathcal{F}$ and $\log \mathcal{N}_\infty(\mathcal{F}; \epsilon; x^{(1)}, \dots, x^{(m)}) \le C_\mathcal{F}/\eps^2$ for all $x^{(1)}, \ldots, x^{(m)} \in \mathcal{X}^m$. Then for any $\delta > 0$, with probability at least $1 - \delta$, simultaneously for all $f \in \mathcal{F}$,
\[
		\left|\risk(f; \mathcal{D}) - \hat{\risk}\left(f; (x^{(i)}, y^{(i)})_{i=1}^m\right)\right| \le 4 c L  \sqrt{\frac{C_\mathcal{F}}{m}} \left(1 + \log \left(A \sqrt{m/C_\mathcal{F}}\right)\right)
		+ 2 b \sqrt{\frac{\log (1/\delta)}{2m}},
	\]
	for some constant $c > 0$.
\end{lemma}
\subsection{Useful lemmas}


\begin{lemma} \label{lem:jacobian}
Consider function $f: \RR^d \rightarrow \Delta^{d-1}$ such that the Jacobian of the function satisfies $\norm{J~f(\theta)}_{1,1} \le c_f$ for all $\theta \in \RR^d$, then for any vectors $\theta_1, \theta_2 \in \RR^p$,
\[\|f(\theta_1) - f(\theta_2)\|_1 \le c_f\|\theta_1- \theta_2\|_\infty.\]
\end{lemma}
\begin{proof}
By the fundamental theorem of calculus applied to $g(t) = f(t\theta_1 + (1-t)\theta_2)$, followed by a change of variables:
\[
f(\theta_1) - f(\theta_2) = \left(\int_0^1 J\left(t\theta_1 + (1-t)\theta_2\right) dt\right)(\theta_1 - \theta_2),
\]
We have
\begin{align*}
\norm{f(\theta_1) - f(\theta_2)}_1 &=\norm{\int_0^1 J\left(t\theta_1 + (1-t)\theta_2\right)(\theta_1 - \theta_2) dt}_1 \\
\shortintertext{By Jensen's inequality:}&\le \int_0^1 \norm{J\left(t\theta_1 +(1-t)\theta_2\right)(\theta_1 - \theta_2)}_1 dt\\
\shortintertext{Using $\norm{Ax}_1 \le \norm{A}_{1,1}\norm{x}_\infty$:}&\le \int_0^1 \norm{J\left(t\theta_1 +(1-t)\theta_2\right)}_{1,1}\norm{\theta_1 - \theta_2}_\infty dt\\
\shortintertext{By assumption on the Jacobian:}&\le c_f\norm{\theta_1 - \theta_2}_\infty.
\end{align*}
\end{proof}

\begin{corollary} \label{lem:softmax}
For vectors $\theta_1, \theta_2 \in \RR^p$, $\|\softmax(\theta_1) - \softmax(\theta_2)\|_1 \le 2\|\theta_1- \theta_2\|_\infty$.
\end{corollary}
\begin{proof}
Observe that for $\softmax$, the Jacobian satisfies:
\[
J(\theta) = \diag(\softmax(\theta)) - \softmax(\theta)\softmax(\theta)^\top.
\]
We have for all $\theta, h$,
\begin{align*}
    \norm{J(\theta)}_{1,1}&= \sum_{i=1}^p \sum_{j=1}^p \left|\softmax(\theta)_i(\mathbbm{1}[i=j] - \softmax(\theta)_j)\right| \\
    & = \sum_{i=1}^p\softmax(\theta)_i\left(1 - \softmax(\theta)_i + \sum_{j \ne i} \softmax(\theta)_j\right) \\
    & = 2\sum_{i=1}^p\softmax(\theta)_i\left(1 - \softmax(\theta)_i\right) \\
    &\le 2.
\end{align*}
Combining the above with Lemma \ref{lem:jacobian} gives the desired result.
\end{proof}


\begin{lemma}\label{lem:optimization}
For $\alpha_i, \beta_i \ge 0$, the solution to the following optimization
\begin{align*}
\min_{x_1, \ldots, x_n} &\sum_{i=1}^n \frac{\alpha_i}{x_i^2} \\
\text{subject to } &\sum_{i=1}^n \beta_i x_i = C
\end{align*}
is $\frac{\gamma^3}{C^2}$ and is achieved at $x_i = \frac{C}{\gamma}\left(\frac{\alpha_i}{\beta_i}\right)^{1/3}$ where $\gamma = \sum_{i=1}^n \alpha_i^{1/3}\beta_i^{\frac{2}{3}}$.
\end{lemma}
\begin{proof}
The proof follows by a standard Lagrangian analysis.
\end{proof}

\begin{lemma}[Contractivity of $\Pi_\mathsf{norm}$]\label{lem:picontract}
Let $\Pi_\mathsf{norm}$ be the projection operator onto the unit norm ball. For any vectors $u, v$, we have $\|\Pi_\mathsf{norm}(u) - \Pi_\mathsf{norm}(v)\| \le \|u - v\|$. 
\end{lemma}
\begin{proof}
If $u, v$ are both in the unit ball then this follows trivially. Let us assume that $\|u\| \ge \|v\|$ and $\|u\| \ge 1$ WLOG.
First suppose $\|v\| \leq 1$. Let $B_V^{(1)} = \alpha u$ be the projection of $v$ in the direction of $u$, and let $B_V^2 = v - B_V^{(1)}$. Then
\begin{align*}
\|\Pi_\mathsf{norm}(u) - \Pi_\mathsf{norm}(v) \|^2
&= \|u/\|u\| - v\|^2 \\
&= \|u/\|u\| - (\alpha u + B_V^2)\|^2 \\
&= \|(\|u\|^{-1} - \alpha)u - B_V^2\|^2 \\
&= (\|u\|^{-1} - \alpha)^2\|u\|^2 + \|B_V^2\|^2 \\
&\leq (1-\alpha^2)\|u\|^2 + \|B_V^2\|^2 & \text{since $\|u\|^{-1} < \alpha < 1$}\\
&= \|u - (\alpha u + B_V^2)\|^2 \\
&= \|u - v\|^2
\end{align*}

If $\|v\| > 1$, then
\[
\|\Pi_\mathsf{norm}(u) - \Pi_\mathsf{norm}(v) \| = \|\Pi_\mathsf{norm}(u/\|v\|) - \Pi_\mathsf{norm}(v/\|v\|) \| \leq \|u/\|v\| - v/\|v\|\| < \|u - v \|.
\]
where the second-to-last inequality follows from the $\|v\| < 1$ case.
\end{proof}

\begin{lemma}[\citet{zhang2002covering}, Theorem~4]\label{lem:zhang}
 Let $\mathcal{V}: \{v : v \in \R^{d_1}, \|v\| \leq B_1\}$ and $\mathcal{F}_\textrm{linear} = \{x \mapsto v^\top x: v \in \mathcal{V}\}$. For any $\delta > 0$ and $x^{(1)}, \dots, x^{(N)}$ satisfying $\|x^{(i)}\| \leq B_2$ $\forall i$,
\[
\log \mathcal{N}_\infty(\mathcal{F}_\textrm{linear}; \eps; x^{(1)}, \cdots, x^{(N)}) \le 36 \frac{ B_1^2 B_2^2 }{\eps^2}\log(2\lceil4B_1 B_2/\eps + 2\rceil N + 1).
\]
\end{lemma}

\subsection{Pseudo-dimension lower bound}
\label{subsec:pseudodim}

Because the number of parameters in a Transformer self-attention head is $O(d^2)$, with no dependence on $T$, one might guess that the capacity of the class of these heads does not need to grow with the context length $T$. But parameter counting can be misleading---for example, the class $\{ x \in \R \mapsto \mathrm{sign}(\sin(\alpha x))\}_{\alpha \in \R}$ has a single parameter but infinite VC-dimension. We observe that, when the weight norms are unbounded, the pseudo-dimension of the class of Transformer self-attention heads does grow with $T$, at least logarithmically, even when the embedding dimension is as small as 3.

Recall that a Transformer self-attention head is of the form
\[
	f_{\tfHead}(X; W_V, W_{QK}) :=
	\sigma\left(W_V^\top X^\top \mathsf{softmax}\left(XW_{QK}^\top x_\tau\right)\right)
\]
Let $\sigma: \mathbb{R} \to \mathbb{R}$ be any activation function that satisfies the following condition: there is a constant $c \in \mathbb{R}$ such that for $a>\frac{1}{2}$ and $b<\frac{1}{2}$, we have $\sigma(a) > c$, $\sigma(b) < c$.

The \emph{pseudo-dimension} of a concept class $\mathcal{F}$ is defined as
\[
\mathrm{Pdim}(\mathcal{F}) := \max_{c \in \mathbb{R}} \mathrm{VCdim}(\{h(x) = \sign(f(x) - c) : f \in \mathcal{F}\})
\]

\begin{proof}[Proof of Proposition~\ref{prop:shattering}]
	
	For the purposes of this proof, we will treat $x_\tau$ as a fixed vector---so either it is not treated as part of the input, or it is set to the same value for all of the inputs. Thus, $W_{QK}^\top x_\tau$ is a fixed vector, which we call $w_{QK}$. Moreover, we will set $W_V$ to be a $d \times 1$ matrix, so we will treat it as a vector $w_V \in \mathbb{R}^d$. Thus, we will be dealing with the following restricted concept class of unbounded Transformer self-attention heads:
	
	\[
	f_{\tfHead}(X; w_V, w_{QK}) :=
	\sigma\left(w_V^\top X^\top \mathsf{softmax}\left(Xw_{QK} \right)\right)
	\]
	for $w_V, w_{QK} \in \mathbb{R}^d$.
	
	For simplicity, we consider the case where $T$ is a power of 2. We will construct a set of $\log T$ inputs $\{X^{(i)}\}_{i=1}^{\log T}$ in $\R^T$ that are shattered by $\mathcal{F}_{\mathsf{tf-head-unbounded}}$.
	
	In particular, indexing $t = 0, \dots, T-1$, let
	\[
	X_t^{(i)} := (\cos(2\pi t/T), \sin(2\pi t/T), \mathrm{bin}(t)_i)
	\]
	where $\mathrm{bin}(t)_i$ is the $i$th bit of the binary expansion of the integer $t$ (padded with 0s in the front such that the expansion is length $\log T$). We can think of the first two coordinates as a (fixed) ``positional encoding'' which allows for the attention mechanism to select a single position, and the third coordinate as the ``token embedding''. The token embedding is designed such that for each binary vector of length $\log T$, there is a position $t$ such that the vector of values at that position $(X_t^{(i)}[3])_{i=1}^{\log T}$ corresponds to the vector, thus inducing a shattering.
	
	For $s=0, \dots T-1$, let
	\[
	w_{QK}^{(s)} := (T^2 \cos(2\pi t/T), T^2 \sin(2\pi t/T), 0).
	\]
	These attention weights are of sufficiently large magnitude that they ``pick out'' a single coordinate. Also, let
	\[
	w_V^{(s)} := (0,0,1).
	\]
	Then we claim that the set
	
	Observe that
	\[
	{w_{QK}^{(s)}}^\top x_t^{(i)} = T^2 \cos(2\pi(s-t)/T),
	\]
	so we have
	\[
	\mathsf{softmax}(Xw_{QK}^{(s)})[t] =
	\frac{\exp(T^2 \cos(2\pi(s-t)/T))}{\sum_{\tau = 0}^{T-1} \exp(T^2 \cos(2\pi(s-\tau)/T))} =
	\frac{\exp(T^2 \cos(2\pi(s-t)/T))}{\sum_{\tau = 0}^{T-1} \exp(T^2 \cos(2\pi \tau/T))}
	\]
	
	Let us bound the denominator using the fact that $\cos(\theta) \leq 1-\frac{2}{\pi^2}\theta^2$ for $\theta \in [0, \pi]$:
	\begin{align*}
		\sum_{\tau = 0}^{T-1} \exp(T^2 \cos(2\pi \tau/T))
		&\leq \exp(T^2) + 2\sum_{\tau = 1}^{\lfloor (T-1)/2 \rfloor} \exp(T^2 \cos(2\pi \tau/T)) \\
		&\leq \exp(T^2) + 2\sum_{\tau = 0}^{\lfloor (T-1)/2 \rfloor} \exp\left(T^2 \left(1-\frac{2}{\pi^2}(2\pi \tau/T)^2\right)\right) \\
		&= e^{T^2} + 2e^{T^2}\sum_{\tau=1}^{\lfloor (T-1)/2 \rfloor} e^{-8\tau^2} \\
		&\leq e^{T^2} + 2e^{T^2} \int_{\rho=0}^{\lfloor (T-1)/2 \rfloor} e^{-8\rho^2} d\rho \\
		&\leq e^{T^2} + 2e^{T^2} \int_{\rho=0}^\infty e^{-8\rho^2} d\rho \\
		&= e^{T^2} + \sqrt{\frac{\pi}{8}}e^{T^2}
	\end{align*}

	Hence, for each $s$,
	\[
	\mathsf{softmax}(Xw_{QK}^{(s)})[s] \geq \frac{1}{1+\sqrt{\frac{\pi}{8}}} > \frac{1}{2}
	\]
	and
	\[
	\sum_{t \neq s} \mathsf{softmax}(Xw_{QK}^{(s)})[t] \leq 1 - \frac{1}{1+\sqrt{\frac{\pi}{8}}} < \frac{1}{2}
	\]
	
	Then we claim that the set $\{f_{\tfHead}(\cdot; w_V^{(s)}, w_{QK}^{(s)})\}_{s =0}^{T-1}$ shatters $\{X^{(i)}\}_{i=1}^{\log T}$:
	
	\begin{align*}
		f_{\tfHead}(X^{(i)}; w_V^{(s)}, w_{QK}^{(s)}) &=
		\sigma\left({w_V^{(s)}}^\top X^\top \mathsf{softmax}\left(X^{(i)}w_{QK}^{(s)} \right)\right) \\
		&= \sigma \left({w_V^{(s)}}^\top\sum_{t=0}^{T-1} \mathsf{softmax}\left(X^{(i)}w_{QK}^{(s)}\right)[t] \cdot \mathrm{bin}(t)_i\right) \\
	\end{align*}
	which is greater than $c$ if $\mathrm{bin}(s)_i = 1$ and is less than $c$ if $\mathrm{bin}(s)_i = 0$, since the $t=s$ term in the sum dominates all the other terms. Thus, the different choices of $s$ induce a shattering.
\end{proof}

\subsection{Covering number upper bounds}

\begin{proof}[Proof of Lemma \ref{lem:lipschitz_general}]
Observe that,
\begin{align*}
&\norm{f_\Head(X, z; \theta_s, \theta_\mathrm{in}) - f_\Head(X, z; \hat{\theta}_s, \hat{\theta}_\mathrm{in}))}\\
&= \left\|\phi_\mathrm{out}\left(\phi_\mathrm{in}(X; \theta_\mathrm{in})^\top \normal(\score(X, z; \theta_s))\right) - \phi_\mathrm{out}\left(\phi_\mathrm{in}(X; \hat{\theta}_\mathrm{in})^\top \normal(\score(X, z; \hat{\theta}_s))\right)\right\|\\
\shortintertext{By $L_{\mathrm{out}}$-Lipschitzness of $\phi_\mathrm{out}$ and bound on $\|w\|$:}
&\le L_{\mathrm{out}} \left\|\phi_\mathrm{in}(X; \theta_\mathrm{in})^\top \normal(\score(X, z; \theta_s)) - \phi_\mathrm{in}(X; \hat{\theta}_\mathrm{in})^\top \normal(\score(X, z; \hat{\theta}_s))\right\|\\
\shortintertext{By triangle inequality:}
&\le L_{\mathrm{out}} \left\|\phi_\mathrm{in}(X; \theta_\mathrm{in})^\top \left(\normal(\score(X, z; \theta_s)) -\normal(\score(X, z; \hat{\theta}_s))\right)\right\|\\
&\quad +L_{\mathrm{out}} \left\|\left(\phi_\mathrm{in}(X; \theta_\mathrm{in}) - \phi_\mathrm{in}(X; \hat{\theta}_\mathrm{in})\right)^\top \normal(\score(X, z; \hat{\theta}_s))\right\|\\
\shortintertext{Using $\|P v\| \le \|P\|_{2, \infty}\|v\|_1$ and $B_\mathrm{in}$-boundedness of $\phi_\mathrm{in}$:}
&\le L_{\mathrm{out}} B_\mathrm{in} \left\|\normal(\score(X, z; \theta_s)) -\normal(\score(X, z; \hat{\theta}_s))\right\|_1\\
&\quad +L_{\mathrm{out}} \left\|\left(\phi_\mathrm{in}(X; \theta_\mathrm{in}) - \phi_\mathrm{in}(X; \hat{\theta}_\mathrm{in})\right)^\top\right\|_{2, \infty} \left\|\normal(\score(X, z; \hat{\theta}_s))\right\|_1\\
\shortintertext{By Lemma \ref{lem:jacobian} and the assumption on $\normal$:}
&\le L_{\mathrm{out}}  C_{\normal}\left\|\phi_\mathrm{in}(X; \theta_\mathrm{in})^\top\right\|_{2, \infty} \left\|\score(X, z; \theta_s) -\score(X, z; \hat{\theta}_s)\right\|_\infty
 +L_{\mathrm{out}} \left\|\left(\phi_\mathrm{in}(X; \theta_\mathrm{in}) - \phi_\mathrm{in}(X; \hat{\theta}_\mathrm{in})\right)^\top\right\|_{2, \infty}\\
\shortintertext{By boundedness of $\phi_\mathrm{in}$ and $\norm{X^\top}_{2,\infty} \le B_X$:}
&\le L_{\mathrm{out}} C_{\normal}B_\mathrm{in} B_X \left\|\score(X, z; \theta_s) -\score(X, z; \hat{\theta}_s)\right\|_\infty + L_{\mathrm{out}}\left\|\left(\phi_\mathrm{in}(X; \theta_\mathrm{in}) - \phi_\mathrm{in}(X; \hat{\theta}_\mathrm{in})\right)^\top\right\|_{2, \infty}.
\end{align*}
\end{proof}

\begin{proof}[Proof of Theorem \ref{lem:covering_head}]
Our goal is to show that for every $\eps > 0$, collection of inputs $(X^{(1)}, z^{(1)}), \dots, (X^{(m)}, z^{(m)})$, there is a cover $\mathcal{C}_\Head$ such that for all $\theta_s \in \Theta_s, \theta_\mathrm{in} \in \Theta_\mathrm{in}$, there is some $(\hat{\theta}_s, \hat{\theta}_\mathrm{in}) \in \mathcal{C}_\Head$ such that $\max_i \norm{f_\Head(X^{(i)}, z^{(i)}; \theta_s, \theta_\mathrm{in}) - f_\Head(X^{(i)}, z^{(i)}; \hat{\theta}_s, \hat{\theta}_\mathrm{in})} \leq \eps$.



Observe that for all $\theta_s, \hat{\theta}_s$,
\[\max_{i \in [m]} \|\score(X^{(i)}, z^{(i)}; \theta_s) - \score(X^{(i)}, z^{(i)}; \hat{\theta}_s)\|_{\infty} = \max_{i \in [m], t \in [T]} \left|\score(x_t^{(i)}, z^{(i)}; \theta_s) - \score(x_t^{(i)}, z^{(i)}; \hat{\theta}_s)\right|.\]
Similarly, for all $\theta_\mathrm{in}, \hat{\theta}_\mathrm{in}$,
\[\max_{i \in [m]}\left\|\left(\phi_\mathrm{in}(X^{(i)}; \theta_\mathrm{in}) - \phi_\mathrm{in}(X^{(i)}; \hat{\theta}_\mathrm{in})\right)^\top\right\|_{2, \infty}  = \max_{i \in [m], t \in [T]} \left\|\phi_\mathrm{in}(x_t^{(i)}; \theta_\mathrm{in}) - \phi_\mathrm{in}(x_t^{(i)}; \hat{\theta}_\mathrm{in})\right\| .\]
This crucially allows us to aggregate over the $i$ and $t$ dimensions together.\footnote{In the case of the Transformer self-attention mechanism, we will obtain $\infty$-norm covering numbers for $\score$ and $\phi_\mathrm{in}$ that have only logarithmic dependence on the number of examples. Because of this aggregation trick, the resulting covering number for the whole layer will have merely logarithmic dependence on the context length $T$.} Therefore, we can consider $\mathcal{N}_\infty$ covers for the above to bound the overall covering number. 

\sloppy Let $\mathcal{C}_\score$ be the $\eps_\score$-cover ($\infty$) for $\mathcal{F}_\score$ over inputs $\left\{(x_t^{(i)}, z^{(i)})\right\}_{i \in [m], t \in [T]}$ of size
\[\mathcal{N}_\infty\left(\mathcal{F}_\score; \eps_\score; \{(x_t^{(i)}, z^{(i)})\}_{i \in [m], t \in [T]}\right).\]
Also, Let $\mathcal{C}_\mathrm{in}$ be the $\eps_\mathrm{in}$-cover ($\infty$) for $\mathcal{F}_\mathrm{in}$ over inputs $\{x_t^{(i)}\}_{i \in [m], t \in [T]}$ of size \[\mathcal{N}_\infty\left(\mathcal{F}_\mathrm{in}; \eps_\mathrm{in}; \{x_t^{(i)}\}_{i \in [m], t \in [T]}; \|\cdot \|_2\right).\]


We are ready to construct the cover for $\mathcal{F}_\Head$. Set $\mathcal{C}_\Head = \{f_\Head(\cdot; \hat{\theta}_s, \hat{\theta}_\mathrm{in}))_{i \in [m]}: \hat{\theta}_s \in \mathcal{C}_\score, \hat{\theta}_\mathrm{in} \in \mathcal{C}_\mathrm{in}\}$. Then for any $\theta_s \in \Theta_s, \theta_\mathrm{in} \in \Theta_\mathrm{in}$, there exists $\hat{\theta}_s, \hat{\theta}_\mathrm{in} \in \mathcal{C}_\Head$, such that for all $i \in [m]$, using Lemma \ref{lem:lipschitz_general}:
\begin{align*}
\norm{f_\Head(X^{(i)}, z^{(i)}; \theta_s, \theta_\mathrm{in}) - f_\Head(X^{(i)}, z^{(i)}; \hat{\theta}_s, \hat{\theta}_\mathrm{in})} &\le C_{\normal} L_{\mathrm{out}} B_\mathrm{in} B_X \eps_\score +L_{\mathrm{out}} \eps_\mathrm{in}.
\end{align*}
The size of the cover we have constructed is,
\begin{align*}
\log |\mathcal{C}_\Head| &= \log |\mathcal{C}_\score| + \log |\mathcal{C}_\mathrm{in}| \\
& = \log \mathcal{N}_\infty\left(\mathcal{F}_\score; \eps_\score; \{(x_t^{(i)}, z^{(i)})\}_{i \in [m], t \in [T]}\right) + \log \mathcal{N}_\infty\left(\mathcal{F}_\mathrm{in}; \eps_\mathrm{in}; \{x_t^{(i)}\}_{i \in [m], t \in [T]}; \|\cdot \|_2\right)
\end{align*}
and we are done.
\end{proof}

\begin{proof}[Proof of Corollary~\ref{lem:covering}]

By Theorem~\ref{lem:covering_head}, the covering number of $\mathcal{F}_\tfHead$ satisfies
\begin{align*}
&\log \mathcal{N}_\infty\left(\mathcal{F}_\tfHead; \epsilon; \left\{(X^{(i)}, z^{(i)})\right\}_{i=1}^m\right)\\
&\le \inf_{\alpha \in [0,1]}  \left[\log \mathcal{N}_\infty\left(\mathcal{F}_{QK}; \frac{\alpha\eps}{2L_\sigma B_V B_X }; \{(x_t^{(i)}, z^{(i)})\}_{i \in [m], t \in [T]}\right)\right.\\
&\left. \quad+ \log \mathcal{N}_\infty\left(\mathcal{F}_V; \frac{(1- \alpha)\eps}{L_\sigma}; \{x_t^{(i)}\}_{i \in [m], t \in [T]}; \|\cdot \|_2\right)\right].
\end{align*}
where we have used the fact that for a scalar-output Transformer layer:
\begin{itemize}
\item $\softmax$ satisfies the Jacobian assumption with $C_\softmax =2$ using Corollary \ref{lem:softmax}. 
\item $L_\mathrm{out}$ is the Lipschitz constant of $\sigma$: $L_\sigma$.
\item $B_\mathrm{in}$ is a bound on the norm of $W_V^\top x$ with respect to norm of $x$: $B_V$.
\end{itemize}
By Lemma~\ref{lem:2-inf-cover}, for any $\eps_{QK}, \eps_V > 0$:
\[
\log \mathcal{N}_\infty\left(\mathcal{F}_{QK}; \eps_{QK}; \{(x_t^{(i)}, z^{(i)})\}_{i \in [m], t \in [T]}\right) \lesssim
\frac{(B_{QK}^{2,1} B_X)^2 \log(dmT)}{\eps_{QK}^2}
\]
\[
\log \mathcal{N}_\infty\left(\mathcal{F}_V; \eps_V; \{(x_t^{(i)}, z^{(i)})\}_{i \in [m], t \in [T]}; \|\cdot\|_2 \right) \lesssim
\frac{(B_V^{2,1} B_X)^2 \log(dmT)}{\eps_V^2}
\]
since $W_{QK},W_V \in \R^{d \times d}$ ($k=d$). We want to choose $\eps_{QK}$ and $\eps_V$ to minimize the sum of the above two terms, subject to 
\[
2L_\sigma  B_V B_X \eps_{QK} + L_\sigma \eps_V \leq \eps.
\]
By Lemma~\ref{lem:optimization}, the solution to this optimization leads to an optimal bound of:
\[
\log \mathcal{N}_\infty(\mathcal{F}_\tfHead; \epsilon; X^{(1)}, \dots, X^{(M)})
\lesssim (L_\sigma B_X)^2 \cdot \frac{\left( ( B_V^{2,1} )^\frac{2}{3} + ( B_{QK}^{2,1} B_V B_X)^\frac{2}{3} \right)^3 }{\eps^2}\cdot \log(dmT).
\]
\end{proof}

\begin{proof}[Proof of Lemma \ref{lem:2-inf-cover}]
Our approach will be to construct a cover by decomposing the problem into two separate cover problems, (1) $\ell_2$-cover over the possible norms of the rows of $W$, and (2) $\ell_\infty$-cover of the set $\mathcal{W}$ constrained to the norms dictated by the first cover. More formally, let us define:
\begin{align*}
\mathcal{W}_{:,2} &= \left\{\begin{bmatrix} \|w_1\|\\\vdots\\\|w_{d_1}\|\end{bmatrix} :W = \begin{bmatrix} w_1\\\vdots\\w_{d_1}\end{bmatrix} \in \mathcal{W}\right\} = \left\{v \in \mathbb{R}^{d_1} : \|v\|_1 \le B_W\right\} \text{, and}\\
\mathcal{F}_v &= \left\{x \rightarrow Wx : W = \begin{bmatrix} w_1\\\vdots\\w_{d_1}\end{bmatrix} \in \mathbb{R}^{d_1 \times d_2}, \forall i~ \|w_i\|\le v_i \right\}.
\end{align*}
Denote $\mathcal{C}_{\mathcal{W}}$ to be the $\eps_1$-cover (in terms of $\ell_2$-norm) for $\mathcal{W}_{:,2}$\footnote{Here the cover is for a set and not a function.}, and for any $v \in \mathbb{R}^{d_1}$, denote $\mathcal{C}_{v}:= \mathcal{N_\infty}\left(\mathcal{F}_v, \eps_2; x^{(1)}, \dots, x^{(N)}; \|\cdot\|_2\right)$. We will set $\eps_1$ and $\eps_2$ later.

We will first show that $\mathcal{C}:= \{f : v \in \mathcal{C}_{\mathcal{W}}, f \in \mathcal{F}_v\}$ is an $(\eps_2 + B_X\eps_1)$-cover (in terms of $\ell_\infty$) for $\mathcal{F}$. Consider $f \in \mathcal{F}$ parameterized by some $W = \begin{bmatrix} w_1\\\vdots\\w_{d_1}\end{bmatrix}$. Since $W \in \mathcal{W}$, we know that there is a $\bar{v} \in \mathcal{C}_\mathcal{W}$, such that,
\[
\sqrt{\sum_{i=1}^{d_1}(\|w_i\| - \bar{v}_i)^2} \le \eps_1.
\]
Define $\bar{W} = \begin{bmatrix} \bar{v}_1 \cdot \frac{w_1}{\|w_1\|}\\\vdots\\ \bar{v}_{d_1} \cdot \frac{w_{d_1}}{\|w_{d_1}\|}\end{bmatrix}$. Then there exists $f \in \mathcal{F}_{\bar{v}} \subseteq \mathcal{C}$ such that 
\[
\max_{i \in [N]} \|\bar{W}x^{(i)} - f(x^{(i)})\| \le \eps_2.
\]
Now we have, for all $i \in [N]$,
\begin{align*}
     \|Wx^{(i)} - f(x^{(i)})\| &\le \|\bar{W}x^{(i)} - f(x^{(i)})\| + \|(\bar{W} - W)x^{(i)}\|\\
     &\le \eps_2 + \sqrt{\sum_{j=1}^{d_1}(\bar{v}_j -\|w_j\|)^2 \left( \frac{w_j \cdot x^{(j)}}{\|w_j\|^2}\right)^2}\\
     &\le \eps_2 + B_X\sqrt{\sum_{j=1}^{d_1}\left(\bar{v}_j -\|w_j\|)^2\right)}\\
     &\le \eps_2 + B_X\eps_1.
\end{align*}
Thus, we get the desired cover. Note that the size of the cover satisfies
\[
\log|\mathcal{C}| \le \log|\mathcal{C}_W| + \max_{v \in \mathcal{C}_W} \log |\mathcal{C}_v|.
\]
Now we need to construct the sub-covers and bound the size of $\mathcal{C}$. Let us first construct the cover $\mathcal{C}_W$. Using Maurey's sparsification (see Theorem 3 in \cite{zhang2002covering}), we can find a proper cover of $\mathcal{C}_W$ which satisfies,
\[
\log|\mathcal{C}_W| \le \frac{B_W^2}{\eps_1^2} \log(2d_1 + 1).
\]
Let us now construct $\mathcal{C}_v$ for a fixed $v$. The approach will be to cover each of the rows of $W$ independently, treating each as specifying a linear function from $\R^{d_2} \to \R$. By Lemma~\ref{lem:zhang}, letting $\mathcal{Z}(b): \{z : z \in \R^{d_2}, \|z\| \leq b\}$ and $\mathcal{F}_\textrm{linear}(b) = \{x \mapsto z^\top x: z \in \mathcal{Z}(b)\}$, for any $\delta > 0$
\[
\log \mathcal{N}_\infty(\mathcal{F}_\textrm{linear}(b); \delta; x^{(1)}, \cdots, x^{(N)}) \leq \frac{c b^2 B_X^2 \log((1+bB_X/\delta)N)}{\delta^2}.
\]
In fact the cover, which we denote by $\bar{\mathcal{F}}_\mathrm{linear}(b;\delta)$, is proper: $\bar{\mathcal{F}}_\mathrm{linear}(b;\delta) = \{x \mapsto \bar{z}^\top x: \bar{z} \in \bar{Z}\}$ for some finite subset $\bar{Z} \subset \mathcal{Z}(b)$. Then the cover for the matrix can be constructed as,
\[
\mathcal{C}_v = \left\{x \mapsto \bar{Z}x: \bar{Z} = \begin{bmatrix} \bar{z}_1\\\vdots\\\bar{z}_{d_1}\end{bmatrix}: \bar{z}_i \in \hat{\mathcal{F}}_\mathrm{linear}\left(v_i; \eps_2\sqrt{\frac{ v_i}{\|v\|_1}}\right) \ \text{for all}\ i \right\}.
\]
Observe that this forms a cover for $\mathcal{F}_v$. For any $f \in \mathcal{F}_v$ parameterized by $W$, let $\bar{w}_i$ be the closest element in the corresponding row covers, then we have
\[
\|Wx^{(i)} - \bar{W}x^{(i)}\| = \sqrt{\sum_{j=1}^{d_2}(w_i^\top x^{(i)} - \bar{w}_i^\top x^{(i)})^2} \le \sqrt{\sum_{j=1}^{d_2}\eps_2^2\cdot\frac{ v_i}{\|v\|_1}} = \eps_2.
\]
\end{proof}
Note that the size of $\mathcal{C}_v$ can be bounded as,
\begin{align*}
\log|\mathcal{C}_v| &= \sum_{i=1}^{d_2}\log \mathcal{N}_\infty\left(\mathcal{F}_\textrm{linear}(v_i); \eps_2\sqrt{\frac{ v_i}{\|v\|_1}}; x^{(1)}, \cdots, x^{(N)}\right)\\
&\le \sum_{i=1}^{d_2} \frac{c v_i^2 \|v\|_1B_X^2 \log((1+\|v\|_1B_X/\eps_2)N)}{\eps_2^2 v_i}\\
&= \frac{c\|v\|_1^2B_X^2 \log((1+\|v\|_1B_X/\eps_2)N)}{\eps_2^2} \le \frac{cB_W^2B_X^2 \log((1+B_WB_X/\eps_2)N)}{\eps_2^2}.
\end{align*}
Combining the above and setting $\eps_1, \eps_2$ appropriately, we get
\[
\log|\mathcal{C}| \le \frac{B_W^2B_X^2}{\eps^2} \log(d_1N).
\]

\subsection{Capacity with positional embeddings}
\label{sec:positional}
\newcommand{\tfPos}{\mathsf{tf\text{-}pos}}
Since the Transformer architecture is permutation invariant for all $t \ne \tau$, positional embeddings (fixed or trainable) are typically added to the inputs to distinguish the different positions of the tokens. These positional embeddings are matrices $P \in \RR^{T \times d}$ such that $P = [p_1 \ldots p_T]^\top$ for $p_i \in \RR^d$. Accounting for the positional embeddings as input, a single Transformer attention head can be expressed as:
\[
f_\tfPos(X, P; W_V, W_{QK}) :=  \sigma\left(W_V^\top (X + P)^\top \mathsf{softmax}\left((X + P)W_{QK}^\top (x_\tau + p_\tau)\right)\right).
\]
For a fixed positional embedding $P$, let us define
\[
\mathcal{F}_\tfPos(P):= \{ X \rightarrow f_\tfPos(X,P;W_V, W_{QK}): \|W_V\|_{2,1} \le B_V^{2,1}, \|W_{V}\| \le B_V, \|W_{QK}^\top\|_{2,1} \le B_{QK}^{2,1}\} \]. 
Position embedding just impacts the input into the covering bound argument which effects the bound in terms of the $\norm{P^\top}_{2, \infty}$ as given below,
\begin{lemma}\label{lem:lipschitz_pos} For all $X^{(1)}, \ldots, X^{(m)} \in \mathbb{R}^{T \times d}$ such that $\norm{{X^{(i)}}^\top}_{2, \infty} \le B_X$ for all $i \in [m]$, and $P\in \mathbb{R}^{T \times d}$ such that $\|P^\top\|_{2, \infty} \le B_P$, the covering number of $\mathcal{F}_\tfPos(P)$ satisfies
\[
\log \mathcal{N}_\infty(\mathcal{F}_\tfPos(P); \epsilon; X^{(1)}, \dots, X^{(m)}, \| \cdot \|_2)
\lesssim (L_\sigma  (B_X + B_P))^2 \cdot \frac{\left( ( B_V^{2,1})^\frac{2}{3} + (2 B_{QK}^{2,1} B_V(B_X + B_P) )^\frac{2}{3} \right)^3 }{\eps^2} \cdot ~\log(dmT).
\]
\end{lemma}
\begin{proof}
Observe that $f_\tfPos(X, P; W_V, W_{QK}) = f_\tfHead(X + P; W_V, W_{QK})$. Thus we have,
\[
\log \mathcal{N}_\infty\left(\mathcal{F}_\tfPos(P); \epsilon; \left\{(X^{(i)})\right\}_{i=1}^m, \| \cdot \|_2\right) =\log \mathcal{N}_\infty\left(\mathcal{F}_\tfHead; \epsilon; \left\{X^{(i)} + P\right\}_{i=1}^m, \| \cdot \|_2\right).
\]
For all $i \in [m]$, $\norm{(X^{(i)} + P)^\top}_{2, \infty} \le \norm{{X^{(i)}}^\top}_{2, \infty} + \norm{P^\top}_{2, \infty} \le B_X + B_P$. Therefore, using Corollary \ref{lem:covering}, we get the desired result.
\end{proof}
Therefore our bounds go through for fixed positional embeddings. If we were to train the embeddings, we would need a much finer cover on the embeddings which could incur a $T$ dependence.
\subsection{Capacity of multiple parallel heads}
\label{sec:multihead}

In virtually all practical applications of Transformers since their inception, instead of using one set of weights for an attention head, there are parallel attention heads, which have separate identically-shaped parameters; their outputs are concatenated. For the purposes of this analysis, suppose we have 
\[
f_\tfMulti\left(X; \left\{W^{[h]}_V, W^{[h]}_{QK}\right\}_{h=1}^{H}\right) :=  \sum_{h=1}^H f_\tfHead\left(X; W^{[h]}_V, W^{[h]}_{QK}\right).
\]
Let us define the class of multi-head self-attention with $H$ heads as
\begin{align*}
\mathcal{F}_\tfMulti := \Big\{X &\mapsto f_\tfMulti\left(X; \left\{W^{[h]}_V, W^{[h]}_{QK}\right\}_{h=1}^{H}\right): \\ &\forall h \in [H],\left\|{W^{[h]}_V}\right\|_{2,1} \le {B^{2,1}_V}^{[h]}, \left\|W^{[h]}_{V}\right\| \le B^{[h]}_V,\left\|{W^{[h]}_{QK}}^\top\right\|_{2,1} \le {B^{2,1}_{QK}}^{[h]}  \Big\}.
\end{align*}

\begin{lemma}For all $X^{(1)}, \ldots, X^{(m)} \in \mathbb{R}^{T \times d}$ such that $\norm{{X^{(i)}}^\top}_{2, \infty} \le B_X$ for all $i \in [m]$, the covering number of $\mathcal{F}_\tfMulti$ satisfies
\[
\log \mathcal{N}_\infty(\mathcal{F}_\tfMulti; \epsilon; X^{(1)}, \dots, X^{(m)}, \| \cdot \|_2)
\lesssim (L_\sigma B_X)^2 \cdot \frac{\left( \sum_{h=1}^H ({B_V^{2,1}}^{[h]} )^\frac{2}{3} + (2 {B_{QK}^{2,1}}^{[h]} B_V^{[h]})^\frac{2}{3} \right)^3}{\eps^2}\cdot\log(dmT).
\]
\end{lemma}
\begin{proof}
For all $h \in [H]$, let $\mathcal{C}_h$ be an $\eps_h$-covering of $\mathcal{F}_\tfHead$ with weight bounds corresponding to head $h$. Since $f_\tfMulti\left(X; \left\{W^{[h]}_V, W^{[h]}_{QK}\right\}_{h=1}^{H}\right) =  \sum_{h=1}^H f_\tfHead\left(X; W^{[h]}_V, W^{[h]}_{QK}\right)$, we have $\mathcal{C} := \mathcal{C}_1 \times \ldots \times \mathcal{C}_H$\footnote{Here, $\times$ denotes the Cartesian product: the functions obtained by using the every combination of parameters of each individual cover.} is an $\left(\sum_{h=1}^H \eps_h\right)$-covering for $\mathcal{F}_\tfMulti$. Using Corollary \ref{lem:covering} (and optimizing for $\eps_h$ using Lemma \ref{lem:optimization}, by breaking them into individual errors for each head), we have
\[
\log|\mathcal{C}| = \sum_{h=1}^H \log |\mathcal{C}_h| \le \sum_{h=1}^H \le (L_\sigma B_X)^2 \cdot \frac{\left( \sum_{h=1}^H ({B_V^{2,1}}^{[h]} )^\frac{2}{3} + (2 {B_{QK}^{2,1}}^{[h]} B_V^{[h]})^\frac{2}{3} \right)^3}{\eps^2}\cdot\log(dmT).
\]
\end{proof}
To see the dependence on $H$, consider the setting where the weight bounds are the same for each head (dropping the $[h]$ subscript), then we get,
\[
\log \mathcal{N}_\infty(\mathcal{F}_\tfMulti; \epsilon; X^{(1)}, \dots, X^{(m)}, \| \cdot \|_2)
\lesssim (L_\sigma B_X)^2 \cdot H^3 \cdot \frac{\left(({B_V^{2,1}} )^\frac{2}{3} + (2 {B_{QK}^{2,1}} B_V)^\frac{2}{3} \right)^3}{\eps^2}\cdot\log(dmT).
\]

\subsection{Capacity of multi-layer Transformers} \label{subsec:deeptf}
This section analyzes the capacity of an $L$-layer Transformer. Let us denote the weights of layer $i$ by $W^{(i)} := \left\{W_Q^{(i)}, {W_K^{(i)}}, W_V^{(i)}, W_C^{(i)}\right\}$ such that $ \norm{W_K^{(i)}{W_Q^{(i)}}^\top}_2 \le B_{QK}^{(i)}, \norm{W_V^{(i)}}_2 \le B_V^{(i)},\norm{W_C^{(i)}}_2 \le B_C^{(i)}$ and $ \norm{{W_K^{(i)}}^\top{W_Q^{(i)}}}_{2,1} \le {B^{2,1}_{QK}}^{(i)}, \norm{W_V^{(i)}}_{2,1} \le {B^{2,1}_V}^{(i)}$ and $\norm{{W_C^{(i)}}}_{2,1} \le {B^{2,1}_C}^{(i)}$. Let us further denote the set of weights up to layer $i$ by $W^{1:i} = (W^{(1)}, \ldots, W^{i-1})$. Let the input representation of layer $i$ be $g_\tfHead^{(i)}(X; W^{1:i})$. We inductively define $g$ with $g_\tfHead^{(1)}(X; W^{1:1}) = X$
\begin{align*}
g_\tfHead^{(i+1)}\left(X; W^{1:i+1}\right) = \Pi_\mathsf{norm}\left(\sigma\left(\Pi_\mathsf{norm}\left( f\left(g_\tfHead^{(i)}\left(X; W^{1:i}\right); W^{(i)}\right)\right)\right)W_C^{(i)}\right) \text{ with}\\
f\left(Z; \{W_Q, W_K, W_V, W_C\}\right)  = \mathsf{RowSoftmax}\left(Z W_Q \left( Z W_K\right)^\top \right) Z W_V,
\end{align*}
where $\Pi_\mathsf{norm}$ is applied row-wise. Our final output is $g_\tfScalar(X;W^{1:L+1}, w) = w^\top g_\tfHead^{(L)}\left(X; W^{1:L+1}\right)\CLS$ for $\|w\| \le B_w$.

In order to construct a cover, we will first bound the distance between the function $g$ with different weight parameters $W^{1:L+1}$ and $\hat{W}^{1:L+1}$. This bound will depend on the closeness of the parameters which will allow us to construct a cover of the network in an iterative fashion by constructing covers of each layer.

\subsubsection{Lipschitzness of the network}
To bound the Lipschitzness of the network, we will first bound the distance between $f$ with different weights and inputs.
\begin{lemma}[Instantiation of Lemma \ref{lem:lipschitz_general}] \label{lem:lipschitz_mat}
For any $W_K, \hat{W}_K, W_V, \hat{W}_V, W_Q, \hat{W}_Q \in \RR^{d \times k}$, for all $Z \in \mathbb{R}^{T \times d}$ such that $\norm{Z^\top}_{2, \infty} \le 1$,
\begin{align*}
&\norm{\left(f\left(Z; \{W_Q, W_K, W_V, \cdot\}\right) - f\left(Z; \{\hat{W}_Q, \hat{W}_K, \hat{W}_V, \cdot\}\right)\right)^\top }_{2, \infty} \\
&\le  2\norm{W_V}_2 \norm{\left(W_Q W_K^\top - \hat{W}_Q \hat{W}_K^\top\right)Z^\top}_{2,\infty}+ \norm{(W_V - \hat{W}_V)^\top Z^\top}_{2, \infty}
\end{align*}
\end{lemma}
\begin{proof}
Consider a fixed row $\tau$ of the output of the functions,
\begin{align*}
&\norm{f\left(Z; \{W_Q, W_K, W_V, \cdot\}\right)[\tau] - f\left(Z; \{\hat{W}_Q, \hat{W}_K, \hat{W}_V, \cdot\}\right)[\tau]}\\
&= \norm{W_V^\top Z^\top \mathsf{softmax}\left(ZW_K W_Q^\top z_\tau\right) -\hat{W}_V^\top Z^\top \mathsf{softmax}\left(Z\hat{W}_K \hat{W}_Q^\top z_\tau\right)}\\
\shortintertext{By triangle inequality:}&\le \norm{W_V^\top Z^\top \left(\mathsf{softmax}\left(ZW_K W_Q^\top z_\tau\right) - \mathsf{softmax}\left(Z\hat{W}_K \hat{W}_Q^\top z_\tau\right)\right)}\\
&\quad + \norm{(W_V - \hat{W}_V)^\top Z^\top \mathsf{softmax}\left(Z\hat{W}_K \hat{W}_Q^\top z_\tau\right)}\\
\shortintertext{Using $\norm{P v} \le \norm{P}_{2, \infty}\|v\|_1$:}&\le  \norm{W_V^\top Z^\top}_{2, \infty} \norm{\mathsf{softmax}\left(ZW_K W_Q^\top z_\tau\right) - \mathsf{softmax}\left(Z\hat{W}_K \hat{W}_Q^\top z_\tau\right)}_1\\
&\quad +  \norm{(W_V - \hat{W}_V)^\top Z^\top}_{2, \infty} \norm{\mathsf{softmax}\left(Z\hat{W}_K \hat{W}_Q^\top z_\tau\right)}_1\\
\shortintertext{By Corollary \ref{lem:softmax}, $\norm{Z^\top}_{2, \infty} \le 1$, $\|PQ\|_{2,\infty} \leq \|P\|_2 \|Q\|_{2,\infty}$, and $\|P^\top\|_2 = \|P\|_2$:}&\le 2 \norm{W_V}_2 \norm{ZW_K W_Q^\top z_\tau - Z\hat{W}_K \hat{W}_Q^\top z_\tau}_\infty +   \norm{(W_V - \hat{W}_V)^\top Z^\top}_{2, \infty}\\
&\le 2 \norm{W_V}_2 \norm{\left(W_Q W_K^\top - \hat{W}_Q \hat{W}_K^\top\right)Z^\top}_{2,\infty}+ \norm{(W_V - \hat{W}_V)^\top Z^\top}_{2, \infty}.
\end{align*}
\end{proof}

\begin{lemma} \label{lem:lipschitz_X}
For any $W_K, W_V, W_Q \in \RR^{d \times k}$, for all $Z, \hat{Z} \in \mathbb{R}^{T \times d}$ such that $\norm{Z^\top}_{2, \infty} \le 1, \|\hat{Z}^\top\|_{2, \infty} \le 1$,
\begin{align*}
&\norm{\left(f\left(Z; \{W_Q, W_K, W_V, \cdot\}\right) - f\left(\hat{Z}; \{W_Q, W_K, W_V, \cdot\}\right)\right)^\top }_{2, \infty} \\
&\le \norm{W_V}_2\left(1 + 4\norm{W_K W_Q^\top}_2\right)\norm{(Z - \hat{Z})^\top}_{2, \infty}.
\end{align*}
\end{lemma}
\begin{proof}
Consider a fixed row $\tau$ of the output of the functions,
\begin{align*}
&\norm{f\left(Z; \{W_Q, W_K, W_V, \cdot\}\right)[\tau] - f\left(\hat{Z}; \{W_Q, W_K, W_V, \cdot\}\right)[\tau]}\\
&= \norm{W_V^\top Z^\top \mathsf{softmax}\left(ZW_K W_Q^\top z_\tau\right) - W_V^\top \hat{Z}^\top \mathsf{softmax}\left(\hat{Z}W_K W_Q^\top \hat{z}_\tau\right)}\\
\shortintertext{By triangle inequality:}
&\le \norm{W_V^\top \left(Z - \hat{Z}\right)^\top \mathsf{softmax}\left(ZW_K W_Q^\top z_\tau\right)} + \norm{W_V^\top \hat{Z}^\top \left(\mathsf{softmax}\left(ZW_K W_Q^\top z_\tau\right) - \mathsf{softmax}\left(\hat{Z}W_K W_Q^\top \hat{z}_\tau\right)\right)}\\
\shortintertext{Using $\|P v\| \le \|P\|_{2, \infty}\|v\|_1$:}
&\le \norm{W_V^\top \left(Z - \hat{Z}\right)}_{2, \infty} \norm{\mathsf{softmax}\left(ZW_K W_Q^\top z_\tau\right)}_1 \\
&\quad+ \norm{W_V^\top \hat{Z}^\top }_{2, \infty} \norm{\mathsf{softmax}\left(ZW_K W_Q^\top z_\tau\right) - \mathsf{softmax}\left(\hat{Z}W_K W_Q^\top \hat{z}_\tau\right)}_1\\
\shortintertext{By Corollary \ref{lem:softmax}, $\norm{\hat{Z}^\top}_{2, \infty} \le 1$ and $\|PQ\|_{2,\infty} \leq \|P\|_2 \|Q\|_{2,\infty}$:}
&\le \norm{W_V}_2\norm{(Z - \hat{Z})^\top}_{2, \infty} + 2\norm{W_V}_2 \norm{ZW_K W_Q^\top z_\tau - \hat{Z}W_K W_Q^\top \hat{z}_\tau}_\infty\\
\shortintertext{By triangle inequality:}
&\le \norm{W_V}_2\norm{(Z - \hat{Z})^\top}_{2, \infty} + 2\norm{W_V}_2 \left(\norm{(Z - \hat{Z})W_K W_Q^\top z_\tau}_\infty + \norm{\hat{Z}W_K W_Q^\top \left(z_\tau - \hat{z}_\tau\right)}_\infty \right)\\
\shortintertext{Since $\norm{\hat{Z}^\top}_{2, \infty} \le 1$ and $\|Pv\|_\infty \leq \|P^\top\|_{2,\infty} \|v\|$:}
&\le \norm{W_V}_2\left(1 + 4\norm{W_K W_Q^\top}_2\right)\norm{(Z - \hat{Z})^\top}_{2, \infty}.
\end{align*}
\end{proof}

With the above lemmas, we are ready to prove the effect of change of weights on $g$.
\begin{lemma}\label{lem:tflayer}
For any $W_{1}^{i+1}, \hat{W}_{1}^{i+1}$ satisfying the norm constraints, 
\begin{align*}
  &\norm{\left(g_\tfBlock^{(i+1)}(X;W^{1:i+1}) - g_\tfBlock^{(i+1)}(X;\hat{W}^{1:i+1})\right)^\top}_{2, \infty} \\
  &\le  \norm{  \left(W_C^{(i)} - \hat{W}_C^{(i)}\right)^\top \sigma\left(\Pi_\mathsf{norm}\left( f\left(\left(X; \hat{W}^{1:i}\right); \hat{W}^{(i)}\right)\right)\right)^\top}_{2, \infty} \\
  &\quad + L_\sigma B_C^{(i)}B_V^{(i)}\left(1 + 4B_{QK}^{(i)}\right)\norm{\left(g_\tfBlock^{(i)}\left(X; W^{1:i}\right) - g_\tfBlock^{(i)}\left(X; \hat{W}^{1:i}\right)\right)^\top}_{2, \infty}\\
&\quad\quad+ 2{L_\sigma} B_C^{(i)}  B_V^{(i)} \norm{\left({W_Q^{(i)}} {W_K^{(i)}}^\top - \hat{W}_Q^{(i)} {{}\hat{W}_K^{(i)}}^\top\right)g_\tfBlock^{(i)}\left(X; \hat{W}^{1:i}\right)^\top}_{2,\infty} \\
&\quad\quad\quad+ L_\sigma B_C^{(i)}\norm{(W_V - \hat{W}_V)^\top g_\tfBlock^{(i)}\left(X; \hat{W}^{1:i}\right)^\top}_{2, \infty}.
\end{align*}
\end{lemma}
\begin{proof}
Unrolling one layer, we have
\begin{align*}
    &\norm{\left(g_\tfHead^{(i+1)}\left(X; W^{1:i+1}\right) - g_\tfHead^{(i+1)}\left(X; \hat{W}^{1:i+1}\right)\right)^\top}_{2, \infty}\\
    &=\left\| \left(\Pi_\mathsf{norm}\left(\sigma\left(\Pi_\mathsf{norm}\left( f\left(g_\tfHead^{(i)}\left(X; W^{1:i}\right); W^{(i)}\right)\right)\right)W_C^{(i)}\right)\right.\right.\\
   &\qquad\qquad \left.\left. - \Pi_\mathsf{norm}\left(\sigma\left(\Pi_\mathsf{norm}\left( f\left(g_\tfHead^{(i)}\left(X; \hat{W}^{1:i}\right); \hat{W}^{(i)}\right)\right)\right)\hat{W}_C^{(i)}\right)\right)^\top\right\|_{2, \infty}\\
    \shortintertext{Using Lemma \ref{lem:picontract} for each row:}
    &\le \norm{ {W_C^{(i)}}^\top\sigma\left(\Pi_\mathsf{norm}\left( f\left(g_\tfHead^{(i)}\left(X; W^{1:i}\right); W^{(i)}\right)\right)\right)^\top - {{}\hat{W}_C^{(i)}}^\top \sigma\left(\Pi_\mathsf{norm}\left( f\left(g_\tfHead^{(i)}\left(X; \hat{W}^{1:i}\right); \hat{W}^{(i)}\right)\right)\right)}_{2, \infty}\\
    \shortintertext{By triangle inequality for each row:}
    &\le \underbrace{\norm{ {W_C^{(i)}}^\top\left(\sigma\left(\Pi_\mathsf{norm}\left( f\left(g_\tfHead^{(i)}\left(X; W^{1:i}\right); W^{(i)}\right)\right)\right) - \sigma\left(\Pi_\mathsf{norm}\left( f\left(g_\tfHead^{(i)}\left(X; \hat{W}^{1:i}\right); \hat{W}^{(i)}\right)\right)\right)\right)^\top}_{2, \infty}}_{(A)}\\
    &\qquad + \norm{ \left(W_C^{(i)} - \hat{W}_C^{(i)}\right)^\top\sigma\left(\Pi_\mathsf{norm}\left( f\left(g_\tfHead^{(i)}\left(X; \hat{W}^{1:i}\right); \hat{W}^{(i)}\right)\right)\right)^\top}_{2, \infty}.
\end{align*}
Let us focus on term $(A)$.
\begin{align*}
\shortintertext{Bounding the norm per row:}
(A) &\le \norm{W_C^{(i)}}_2\norm{ \sigma\left(\Pi_\mathsf{norm}\left( f\left(g_\tfHead^{(i)}\left(X; W^{1:i}\right); W^{(i)}\right)\right)\right)^\top - \sigma\left(\Pi_\mathsf{norm}\left( f\left(g_\tfHead^{(i)}\left(X; \hat{W}^{1:i}\right); \hat{W}^{(i)}\right)\right)\right)^\top}_{2, \infty}\\
\shortintertext{Since $\sigma$ is
$L_\sigma$-Lipschitz and $\norm{W_C^{(i)}}_2 \le B_C^{(i)}$, for each row:}
&\le L_\sigma B_C^{(i)}\norm{ \Pi_\mathsf{norm}\left( f\left(g_\tfHead^{(i)}\left(X; W^{1:i}\right); W^{(i)}\right)\right)^\top - \Pi_\mathsf{norm}\left( f\left(g\left(X; \hat{W}^{1:i}\right); \hat{W}^{(i)}\right)\right)^\top}_{2, \infty}\\
\shortintertext{Using Lemma \ref{lem:picontract} for each row:}
&\le L_\sigma B_C^{(i)}\norm{ f\left(g_\tfHead^{(i)}\left(X; W^{1:i}\right); W^{(i)}\right)^\top - f\left(g_\tfHead^{(i)}\left(X; \hat{W}^{1:i}\right); \hat{W}^{(i)}\right)^\top}_{2, \infty}\\
\shortintertext{By triangle inequality:}
&\le L_\sigma B_C^{(i)}\norm{ f\left(g_\tfHead^{(i)}\left(X; W^{1:i}\right); W^{(i)}\right)^\top - f\left(g_\tfHead^{(i)}\left(X; \hat{W}^{1:i}\right); W^{(i)}\right)^\top}_{2, \infty}\\
&\quad+ L_\sigma B_C^{(i)}\norm{ f\left(g_\tfHead^{(i)}\left(X; \hat{W}^{1:i}\right); W^{(i)}\right)^\top - f\left(g_\tfHead^{(i)}\left(X; \hat{W}^{1:i}\right); \hat{W}^{(i)}\right)^\top}_{2, \infty}\\
\shortintertext{By Lemma \ref{lem:lipschitz_mat} and \ref{lem:lipschitz_X} and norm bounds on the matrices:}
&\le L_\sigma B_C^{(i)}B_V^{(i)}\left(1 + 4B_{QK}^{(i)}\right)\norm{g_\tfHead^{(i)}\left(X; W^{1:i}\right)^\top - g\left(X; \hat{W}^{1:i}\right)^\top}_{2, \infty}\\
&\quad+ {2L_\sigma} B_C^{(i)}  B_V^{(i)} \norm{\left({W_Q^{(i)}} {W_K^{(i)}}^\top - \hat{W}_Q^{(i)} {{}\hat{W}_K^{(i)}}^\top\right) g_\tfHead^{(i)}\left(X; \hat{W}^{1:i}\right)^\top}_{2,\infty} \\
&\quad+ L_\sigma B_C^{(i)}\norm{(W_V - \hat{W}_V)^\top g_\tfHead^{(i)}\left(X; \hat{W}^{1:i}\right)^\top}_{2, \infty}.
\end{align*}
Combining the above gives us the desired result.
\end{proof}

Lastly, we take account of the last linear weight and observe that,
\begin{lemma} \label{lem:lastlayer}
For any $W^{1:L+1}, \hat{W}^{1:L+1}$ and $w, \hat{w}$,
\begin{align*}
&\left|g_\tfScalar\left(X; W^{1:L+1}, w\right) - g_\tfScalar\left(X; \hat{W}^{1:L+1}, \hat{w}\right) \right| \\
&\le  \norm{w}\norm{g_\tfBlock^{(L+1)}\left(X; W^{1:L+1}\right)_\CLS - g_\tfBlock^{(L+1)}\left(X; \hat{W}^{1:L+1}\right)_\CLS } + \left|(w - \hat{w})^\top g_\tfBlock^{(L+1)}\left(X; \hat{W}^{1:L+1}\right)_\CLS\right|.
\end{align*}
\end{lemma}
\begin{proof}
Observe that,
\begin{align*}
&\left|g_\tfScalar\left(X; W^{1:L+1}, w\right) - g_\tfScalar\left(X; \hat{W}^{1:L+1}, \hat{w}\right) \right|\\ 
&= \left|w^\top g_\tfBlock^{(L+1)}\left(X; W^{1:L+1}\right)_\CLS - \hat{w}^\top g_\tfBlock^{(L+1)}\left(X; \hat{W}^{1:L+1}\right)_\CLS \right|\\
\shortintertext{By triangle inequality:}
&\le \left|w^\top \left(g_\tfBlock^{(L+1)}\left(X; W^{1:L+1}\right)_\CLS - g_\tfBlock^{(L+1)}\left(X; \hat{W}^{1:L+1}\right)_\CLS \right)\right| + \left|(w - \hat{w})^\top g_\tfBlock^{(L+1)}\left(X; \hat{W}^{1:L+1}\right)_\CLS\right|\\
\shortintertext{Bounding the inner product by norms:}
&\le \norm{w}\norm{g_\tfBlock^{(L+1)}\left(X; W^{1:L+1}\right)_\CLS - g_\tfBlock^{(L+1)}\left(X; \hat{W}^{1:L+1}\right)_\CLS } + \left|(w - \hat{w})^\top g_\tfBlock^{(L+1)}\left(X; \hat{W}^{1:L+1}\right)_\CLS\right|.
\end{align*}
\end{proof}
\subsubsection{Constructing the cover}
The cover construction follows the standard recipe of composing covers per layer (as in \cite{bartlett2017spectrally}).
\begin{theorem} \label{thm:deeptf_full}
Let $\mathcal{F}_\tfScalar^{(L)}$ represent the class of functions of $L$-layer Transformer blocks satisfying the norm bounds (specified before) followed by linear layer on the $\CLS$ token. Then, for all $X^{(i)}$
\begin{align*}
&\log \mathcal{N}_\infty(\mathcal{F}_\tfScalar^{(L)}; \epsilon; X^{(1)}, \dots, X^{(m)}, \| \cdot\|_2) \lesssim \\ & \frac{\log(dmT)}{\eps^2} \times
\left(B_w^{\frac{2}{3}} + \sum_{i=1}^L{\alpha_i}^{\frac{2}{3}}\left( {{B^{2,1}_{C}}^{(i)}}^{\frac{2}{3}} + d^{\frac{2}{3}}\left(2L_\sigma B_C^{(i)}  B_V^{(i)} {B^{2,1}_{QK}}^{(i)}\right)^{\frac{2}{3}} + k^{\frac{2}{3}}\left(L_\sigma B_C^{(i)} {B^{2,1}_V}^{(i)}\right)^{\frac{2}{3}}\right)\right)^3
\end{align*}
where $\alpha_i = \prod_{j < i} L_\sigma B_C^{(j)} B_V^{(j)} (1 + 4B_{QK}^{(j)})$.
\end{theorem}
\begin{proof}
Our goal is to show that for every $\eps > 0$, and collection of inputs $X^{(1)}, \dots, X^{(m)}$, there is a cover $\mathcal{C}$ of vectors in $\R^{(m)}$ such that for all $W^{1:L+1}$ and $w$ satisfying the norm bounds, there is some $v \in \mathcal{C}$ such that $\max_i |g_\tfScalar(X^{(i)}; W^{1:L+1}, w) - v| \leq \eps$.

In each layer of the transformer, $W_Q^{(i)}$ and $W_K^{(i)}$ always appear together in the form $W_K^{(i)} {W_Q^{(i)}}^\top$. Therefore, we will overload notation and define $W_{QK}^{(i)} : W_K^{(i)} {W_Q^{(i)}}^\top$. Our cover $\mathcal{C}$ will be proper, consisting of vectors of the form $(g_\tfScalar(X^{(i)}; \hat{W}^{1:L+1}, \hat{w}))_{i \in [m]}$. We will build the cover iteratively by finding finite collections of matrices $\hat{\mathcal{W}}^{1:i}$ for each layer. 

First observe that for any collection of $Z^{(1)}, \ldots, Z^{(m)} \in \RR^{T \times d_1}$, and any $W, \hat{W} \in \RR^{d_1 \times d_2}$,
\[\max_{i \in [m]} \norm{W^\top {Z^{(i)}}^\top - \hat{W}^\top Z{^{(i)}}^\top}_{2, \infty} = \max_{i \in [m], t \in [T]} \norm{W^\top z_t^{(i)} - \hat{W}^\top z_t^{(i)}}.\] 
This crucially allows us to aggregate over the samples and context length. In particular, we can apply Lemma~\ref{lem:2-inf-cover} with the input vectors $(z_t^{(i)})_{i \in [m], t \in [T]}$; a total of $mT$ input vectors. Specifically, for any $\epsilon$ and $\mathcal{W}(d_1, d_2, \alpha):= \{W \in \RR^{d_1 \times d_2}~|~ \|W\|_{2, 1} \le \alpha\}$ with fixed $Z^{(i)}$ satisfying $\norm{{Z^{(i)}}^\top}_{2, \infty} \le 1$, Lemma~\ref{lem:2-inf-cover} gives us such a cover.

\sloppy First let us build a cover for one Transformer layer with inputs $Z^{(1)}, \ldots, Z^{(m)}$. We will begin with creating an $\epsilon_V$-cover $\hat{\mathcal{W}}_V$ for the function class of linear transformations given by $\mathcal{W}_V: \{W \in \RR^{d \times k}, \norm{W}_{2, 1} \le \alpha, \norm{W}_{2} \le s\}$ and $\epsilon_{QK}$-cover $\hat{\mathcal{W}}_{QK}$ for $\mathcal{W}_{QK}:= \{W \in \RR^{d \times d}, \norm{W^\top}_{2,1} \le \beta, \norm{W}_{2} \le r\}$ and inputs $Z^{(1)}, \ldots, Z^{(m)}$. For each pair of $\hat{W}_V \in \hat{\mathcal{W}}_V$ and $\hat{W}_{QK} \in \hat{\mathcal{W}}_{QK}$, we construct an $\epsilon_{C}$-cover $\hat{\mathcal{W}}_C(\hat{W}_V, \hat{W}_{QK})$ for $\mathcal{W}_C: \{W \in \RR^{k \times d}, \norm{W}_{2,1} \le \gamma, \norm{W}_{2} \le c\}$ and inputs $\left\{\sigma\left(\Pi_\mathsf{norm}\left( f\left(Z^{(i)}; \hat{W}_V, \hat{W}_{QK}\right)\right)\right)\right\}_{i=1}^{m}$. Our final cover is \[\hat{\mathcal{W}}:= \left\{(\hat{W}_V, \hat{W}_{QK}, \hat{W}_C) : \hat{W}_V \in \hat{\mathcal{W}}_V, \hat{W}_V \in \hat{\mathcal{W}}_V, \hat{W}_C \in \hat{\mathcal{W}}_C(\hat{W}_V, \hat{W}_{QK}) \right\}.\]

Using Lemma \ref{lem:tflayer}, we can show that $\hat{\mathcal{W}}$ is an $\epsilon$-cover for $g(\cdot; \{W_V, W_{QK}, W_C\})$ and inputs $Z^{(1)}, \ldots, Z^{(m)}$ where
\[
\epsilon = \eps_C + 2L_\sigma c  s \eps_{QK} + L_\sigma c\eps_V.
\]
Using Lemma \ref{lem:2-inf-cover}, the size of the cover is
\begin{align*}
|\hat{\mathcal{W}}| &\le |\hat{\mathcal{W}}_V| |\hat{\mathcal{W}}_{QK}| \max_{\substack{\hat{W}_V \in \hat{\mathcal{W}}_V\\ \hat{W}_{QK} \in \hat{\mathcal{W}}_{QK}}  } \left|\hat{\mathcal{W}}_C(\hat{W}_V, \hat{W}_{QK})\right|\\
\implies \log |\hat{\mathcal{W}}| &\lesssim \left(\frac{\alpha^2 }{\eps_{V}^2} + \frac{ \beta^2}{\eps_{QK}^2} + \frac{\gamma^2}{\eps_{C}^2}\right) \log (dmT).
\end{align*}

\sloppy We are now ready to inductively construct a cover for the deeper network. Suppose we have a $\eps^{(i)}$-cover $\hat{\mathcal{W}}^{1:i}$ for $g(\cdot; W^{1:i})$ on $X^{(1)}, \cdots, X^{(m)}$. We show how to construct an $\eps^{(i+1)}$-cover for $g(\cdot; W^{1:i+1})$. For every element $\hat{W}^{1:i} \in \hat{\mathcal{W}}^{1:i}$ we construct a $\left(\eps_C^{(i)} + 2L_\sigma B_C^{(i)}  B_V^{(i)} \eps_{QK}^{(i)} + L_\sigma B_C^{(i)}\eps_V^{(i)}\right)$-cover $\hat{\mathcal{W}}_i(\hat{W}^{1:i})$ for the transformer layer (as above) on inputs $\left\{g(X^{(j)}; \hat{W}^{1:i})\right\}_{j=1}^{m}$. Consider the cover
\[\hat{\mathcal{W}}^{1:i+1}:= \left\{ (\hat{W}^{1:i}, \hat{W}^{(i)}) : \hat{W}^{1:i} \in  \hat{\mathcal{W}}^{1:i},  \hat{W}^{(i)} \in \hat{\mathcal{W}}_i(\hat{W}^{1:i}) \right\}.\]
By Lemma~\ref{lem:tflayer}, this gives,
\[
\eps^{(i+1)} = L_\sigma B_C^{(i)} B_V^{(i)} (1 + 4B_{QK}^{(i)}) \eps^{(i)} + \eps_C^{(i)} + 2L_\sigma B_C^{(i)}  B_V^{(i)} \eps_{QK}^{(i)} + L_\sigma B_C^{(i)}\eps_V^{(i)}.
\]
The size of the cover is
\begin{align*}
|\hat{\mathcal{W}}^{1:i+1}| &\le |\hat{\mathcal{W}}^{1:i}| \max_{\substack{\hat{W}^{1:i} \in \hat{\mathcal{W}}^{1:i} }} \left|\hat{\mathcal{W}}_i(\hat{W}^{1:i})\right|.
\end{align*}
Inductively applying this, we get
\begin{align*}
\eps^{(L+1)} &= \sum_{i=1}^L \left(\prod_{j < i} L_\sigma B_C^{(j)} B_V^{(j)} (1 + 4B_{QK}^{(j)})\right) \left(\eps_C^{(i)} + 2{L_\sigma} B_C^{(i)}  B_V^{(i)} \eps_{QK}^{(i)} + L_\sigma B_C^{(i)}\eps_V^{(i)}\right)\\
&= \sum_{i=1}^L \alpha_i\left(\eps_C^{(i)} + 2{L_\sigma} B_C^{(i)}  B_V^{(i)} \eps_{QK}^{(i)} + L_\sigma B_C^{(i)}\eps_V^{(i)}\right)
\end{align*}
where $\alpha_i = \prod_{j < i} L_\sigma B_C^{(j)} B_V^{(j)} (1 + 4B_{QK}^{(j)})$.

The size of the cover is
\[
\log \left(|\hat{\mathcal{W}}^{1:L+1}|\right) \le \sum_{i=1}^{L}\left(\frac{{{B^{2,1}_V}^{(i)}}^2 }{{\eps_{V}^{(i)}}^2} + \frac{ {{B^{2,1}_{QK}}^{(i)}}^2}{{\eps_{QK}^{(i)}}^2} + \frac{{ {B^{2,1}_{C}}^{(i)}}^2}{{\eps_{C}^{(i)}}^2}\right) \log (dmT).
\]
 Notice that the layer-norm maintains the norm bound on the inputs. Lastly, we need to cover the linear layer on the $\CLS$ token and compose it with the cover of $g^{1:L}$ (as before). Using Lemma \ref{lem:zhang} and \ref{lem:lastlayer}, we can get the final $\eps$-cover $\mathcal{C}$ with
 \begin{align*}
 \eps = B_w\sum_{i=1}^L\alpha_i\left(\eps_C^{(i)} + 2{L_\sigma} B_C^{(i)}  B_V^{(i)} \eps_{QK}^{(i)} + L_\sigma B_C^{(i)}\eps_V^{(i)}\right) + \eps_w
 \end{align*}
 and size
 \[
 \log |\mathcal{C}| \lesssim \frac{ B_w^2 \log(m)}{\eps_w^2} + \sum_{i=1}^{L}\left(\frac{ {{B^{2,1}_V}^{(i)}}^2 }{{\eps_{V}^{(i)}}^2} + \frac{ {{B^{2,1}_{QK}}^{(i)}}^2}{{\eps_{QK}^{(i)}}^2} + \frac{{ {B^{2,1}_{C}}^{(i)}}^2}{{\eps_{C}^{(i)}}^2}\right) \log (dmT).
 \]
 Using Lemma \ref{lem:optimization}, the size of the cover for fixed $\eps$ gives us the desired result.
\end{proof}



%% file: B-examples-proofs.tex
\section{Sparse function representation via bounded-norm Transformers}
\label{sec:appendix-approx-proofs}

\subsection{Setup}
\label{subsec:appendix-approx-setup}

\paragraph{Reductions from Boolean functions to Transformers.} In order to establish our function approximation results, we must first define a canonical mapping between length-$T$ Boolean strings $b \in \{0, 1\}^T$ and Transformer inputs $X \in \RR^{T \times d}$. The key point (which has also been considered since the inception of the Transformer \citep{vaswani2017attention}, and continues to be a crucial consideration in practice \citep{dosovitskiy2020image}) is that the network's permutation-equivariant symmetry needs to be broken by assigning different embeddings to different indices of $b$. There are several possible natural choices here, which are all of practical interest:
\begin{itemize}
\item \emph{Deterministic positional embeddings.} Fix positional embedding matrices $P \in \RR^{T \times d}, E \in \RR^{\{0,1\} \times d}$, and a special direction $v_\CLS \in \RR^d$, such that the $T+3$ vectors $\{P_{t,:}\}_{t=1}^T \cup \{E_{j,:}\}_{j\in{0,1}} \cup \{v_\CLS\}$ are an approximately orthonormal basis for $\RR^d$ (see below). The input to the Transformer is then $X = E_b + P$, where $E_b \in \RR^{T \times d}$ such that $[E_b]_{t,:} = E_{b_t, :}$ for each $t \in [T]$. In the $f_\tfScalar$ formulation, we choose the auxiliary input $x_\CLS$ to be the constant vector $v_\CLS$. This closely matches applications of Transformers in NLP \citep{vaswani2017attention}.
\item \emph{Trainable positional embeddings.} Like the above, but $P$ is a trainable parameter; we still require approximate orthogonality of $\{E_{j,:}\}_{j\in{0,1}} \cup \{v_\CLS\}$. It is also possible to consider the case where $E$ and $v_\CLS$ are trainable (matching the way token embeddings are trained in practice). This becomes important in the regime of large vocabulary sizes that require embeddings to capture shared information between tokens; however, this is not necessary for our constructions, as we limit our consideration to binary tokens. This simplifies our constructions and improves statistical rates; additionally, it is a popular and well-studied alternative \citep{vaswani2017attention,devlin2018bert,radford2018improving,radford2019language,brown2020language}.
\item \emph{Bag of vectors.} Fix a matrix $V \in \RR^{T \times d}$ with approximately orthogonal rows (like the deterministic $P$), but choose the Transformer input
\[X := V \diag(b).\]
This construction replaces positional embeddings with positional ``indicator vectors'' which can be swapped between any of the Transformer's input positions. It has the advantage of being symmetric with respect to permutation of the Transformer's input positions: it turns out that
\[f_\tfScalar(V \diag(b)) = f_\tfScalar(V \Pi \diag(b)),\]
for any $T\times T$ permutation matrix $\Pi$. It is also the most natural construction when considering the composition of sparse Boolean functions across multiple layers: a layer can output combinations of the basis rows $v_i$ for further function composition, like Boolean gates.
\end{itemize}

\paragraph{Approximately orthonormal basis.}
Each of the Boolean function approximation constructions will rely on a basis set of vectors, which will be used as positional embeddings (or the variable indices in the bag-of-vectors construction). We will fix a set of approximately orthonormal vectors $\{v_i: \norm{v_i} = 1\}_{i=1}^{T'}$ in $\RR^d$: for each $i \neq j$, we have $|v_i^\top v_j| \leq \Delta$. When $\Delta = 0$, the maximal $T'$ for which such a set exists is $d$; for $\Delta \in (0, \frac{1}{2})$, the Johnson-Lindenstrauss lemma \citep{johnson1986extensions} implies that the maximal set of is of size $\exp(\Theta(d\Delta^2))$. For given choices of $d,\Delta$ and a maximal $\{v_1, \ldots, v_{T'}\}$, our construction is valid for contexts of length $T \leq T'$. For the special vectors $e_0, e_1, v_\CLS$, we will assume that these are exactly orthogonal to the $v_i$ and each other, so that the $v_i$ must be a basis in dimension at least $d-3$. This is for clarity only-- it reduces the number of error terms to propagate through the analysis.

\paragraph{Self-attention block.} In each construction (which specifies an input $X \in \RR^{T \times d}$, we will specify the parameters $W_Q, W_K, W_V, W_C, w = e_1$ of a scalar-output Transformer $f_\tfScalar$, which takes an input $X \in \RR^{(T+1) \times d}$; the auxiliary token input will be the constant vector $x_\CLS := v_\CLS \in \RR^d$. The internal activation function $\sigma$ is chosen to be the identity. Summarizing, the functional form of $f_\tfScalar \in \gF_\tfScalar$ in these constructions is
\[ f_\tfScalar(X; W_Q, W_K, W_V, W_C, e_1) = \softmax\pa{v_\CLS^\top W_Q W_K^\top X^\top} X W_V W_C e_1. \]
In the intermediate lemmas, it will also be useful to consider the corresponding attention head output
\[ f_\tfHead(X; W_Q, W_K, W_V, W_C) = \softmax\pa{v_\CLS^\top W_Q W_K^\top X^\top} X W_V W_C, \]
and its projections $f_\tfHead \circ \Pi_{d_\mathrm{proj}}$ onto the first $d_\mathrm{proj}$ coordinates.

\paragraph{Feedforward networks.} We establish some notation for feedforward networks. An $L$-layer feedforward network, with activation function $\sigma : \RR \rightarrow \RR$ and dimensions $d_1, \ldots, d_{L+1}$, is parameterized by weight matrices $W_i \in \RR^{d_{i+1} \times d_i}$, and maps $x \in \RR^{d_1}$ to $y \in \RR^{d_{L+1}}$, by the iterative equations
\[ y_1^\top := \sigma( x^\top W_1 ),\]
\[ y_{i+1}^\top := \sigma( y_i^\top W_i ), \qquad i = 1, \ldots, L-1, \]
\[ f_\mlp(x ; W_1, \ldots, W_L)^\top = y^\top := y_L^\top W_i. \]
When $d_{L+1} = 1$, we will use the notation $w$ instead of $W_L$. It will be convenient to incorporate \emph{bias} weights by introducing an extra input coordinate $W_i \in \RR^{(d_{i+1}+1)\times d_i}$, and augmenting the linear function accordingly:
\[ y_i^\top W_i \mapsto [y_i^\top \; 1] W_i. \]

\paragraph{Self-attention composed with a feedforward network.} The full definition of the Transformer layer composes a self-attention layer ($f_\tfLayer : \RR^{T \times d} \rightarrow \RR^{T \times d}$) with a position-wise feedforward network ($f_\mlp : \RR^d \rightarrow \RR^d$). We will use this combination of modules to establish our function approximation results: $f_\tfLayer$ acts as a sparse bottleneck, while $f_\mlp$ approximates an arbitrary function of the selected coordinates. For our single-layer constructions, it is most convenient to establish notation for a scalar-output Transformer with a feedforward network. To this end, define $\gF_\tfmlp$ to be the function class with the same $\score, \normal, \phi_\mathrm{in}$ functions as in $\gF_\tfScalar$ (thus, the same parameters $W_Q, W_K, W_V$), with identity activation function, but a feedforward neural network replacing the linear $\phi_\mathrm{out}$ and $w$. Concretely, with $L=3$ and the ReLU activation function $(\cdot)_+$, $\gF_\tfmlp$ contains functions of the form
\[f_\tfmlp(X;\theta) =
\pa{ (y^\top W_1)_+ W_2 }_+ w, \]
\[ y = \softmax\pa{v_\CLS^\top W_Q W_K^\top X^\top} X W_V W_C w, \]
with parameters $\theta := (W_Q, W_K, W_V, W_C, W_1, W_2, w)$.

\paragraph{Multiple self-attention heads.} The final component we will need for the function approximation setup is multi-headed self-attention. We will extend the definition of the single-headed $f_\tfHead$ to
\[ f_\tfMulti \pa{ X; \left\{ W_Q^{[h]}, W_K^{[h]}, W_V^{[h]}, W_C^{[h]} \right\}_{h=1}^H } := \sum_{h=1}^H f_\tfHead\pa{ X; W_Q^{[h]}, W_K^{[h]}, W_V^{[h]}, W_C^{[h]} }, \]
and substitute this definition into $f_\tfmlp$ when discussing a multi-headed construction.

\paragraph{Classes and properties of Boolean functions.} We will call a Boolean function $f : \{0, 1\}^T \rightarrow \gY$ $\gI$-\emph{sparse} if it only depends on a fixed subset $\gI \subseteq [T]$ of its inputs:
\[ b_i = b'_i \quad \forall i \in \gI \implies f(b) = f(b'). \]
Overloading notation, if $\gI = s$, we will also call $f$ $s$-sparse.
We will call an $\gI$-sparse Boolean function $f$ \emph{symmetric} if its value is invariant under permutation of the indices in $\gI$:
\[ |\{i\in\gI: b_i = 1\}| = |\{i\in\gI: b_i' = 1\}| \implies f(b) = f(b'). \]
Further, we will call an $\gI$-sparse real-valued symmetric Boolean function $f : \{0, 1\}^T \rightarrow \gY$ \emph{monotone} if $f(b)$ is monotonically increasing in $r := |\{i\in\gI: b_i = 1\}|$. If, for some $\gamma > 0$, it holds that $f(r+1) \geq f(r) + \gamma$ for each $r = 0, \ldots, s-1$, we call $f$ $\gamma$\emph{-strictly monotone}. A vector-valued $\gI$-sparse Boolean function $f : \{0,1\}^T \rightarrow \RR^{d_f}$ is $\gamma$\emph{-injective} if
\[\norm{ f(b) - f(b') }_\infty \geq \gamma\]
for each $b, b'$ that differ at some position $i \in \gI$; $f$ is called $B$-bounded if $\norm{f(b)}_\infty \leq B$ for all $b \in \{0,1\}^T$.

\paragraph{Uniform approximation.}
For some $\eps \geq 0$ and a function $f : \{0, 1\}^T \rightarrow \RR^d$, we say that $\hat{f} \in \mathcal{F}$ $\eps$-uniformly approximates $f$ under the mapping $b \mapsto X(b)$ if
\[ \norm{ \hat f( X(b) ) - f(b) }_\infty \leq \eps, \qquad \forall b \in \{0,1\}^T. \]

\subsection{Results}
\label{subsec:appendix-approx-results}

We give an overview of the function approximation results under each input mapping $b \mapsto X(b)$, as a multi-part proposition:

\begin{proposition}[Sparse variable creation with Transformers]
The function classes $\gF_\tfScalar, \gF_\tfmlp$ contain the following classes of sparse Boolean functions:
\end{proposition}
\begin{itemize}
    \item \emph{Deterministic positional embeddings:} For any $\gI$, $\gF_\tfScalar$ can approximate a particular monotone symmetric $\gI$-sparse $f$, with Transformer weight norm bounds from the real-valued construction in Lemma~\ref{lem:bool-deterministic-pos}. $\gF_\tfmlp$ with $1$ head can exactly represent any symmetric $s$-sparse $f$, with the same bounds on Transformer weight norms, and feedforward network weight norms scaling as $O(\mathrm{poly}(s))$. $\gF_\tfmlp$ with $s$ heads can exactly represent any $s$-sparse $f$, with Transformer weight norm bounds from the vector-valued construction in Lemma~\ref{lem:bool-deterministic-pos}, and feedforward network weight norms scaling as $O(\mathrm{poly}(s) \cdot 2^s)$.
    \item \emph{Trainable positional embeddings:} For any $\gI$, $\gF_\tfScalar$ can approximate a particular monotone symmetric $\gI$-sparse $f$, with positional embedding and Transformer weight norm bounds from the real-valued construction in Lemma~\ref{lem:bool-trainable-pos}. $\gF_\tfmlp$ with $1$ head can exactly represent any symmetric $s$-sparse $f$, with the same bounds on $P$ and Transformer weight norms, and feedforward network weight norms scaling as $O(\mathrm{poly}(s))$. $\gF_\tfmlp$ with $s$ heads can exactly represent any sparse $f$, with $P$ and Transformer weight norm bounds from the vector-valued construction in Lemma~\ref{lem:bool-trainable-pos}, and feedforward network weight norms scaling as $O(\mathrm{poly}(s) \cdot 2^s)$.
    \item \emph{Bag of vectors:} For any $\gI$, $\gF_\tfScalar$ can approximate a particular monotone symmetric $\gI$-sparse $f$, with Transformer weight norms from Lemma~\ref{lem:bool-bag}. $\gF_\tfmlp$ with $1$ head can represent any symmetric $s$-sparse $f$, with the same Transformer weight norm bounds, and feedforward network weight norms scaling as $O(\mathrm{poly}(s))$.
    $\gF_\tfmlp$ with $1$ head can also exactly represent any $s$-sparse $f$, with the same bounds on Transformer weight norms, and feedforward network weight norms scaling as $O(\mathrm{poly}(s) \cdot 2^s)$.
\end{itemize}

The formal statements are obtained by $(\gamma/4)$-uniformly approximating a $\gamma$-strictly monotone or $\gamma$-injective function with self-attention alone (Lemmas~\ref{lem:bool-deterministic-pos}, \ref{lem:bool-trainable-pos}, \ref{lem:bool-bag}), then applying a robust universal function representation construction (Lemmas~\ref{lem:bool-symmetric}, \ref{lem:bool-asymmetric}) appropriately. They are organized as follows:
\begin{lemma}[Deterministic $P$, no MLP]
\label{lem:bool-deterministic-pos}
Suppose $X(b) = P + E_b$ with deterministic $P$. Let $\gI \subseteq [T]$ such that $|\gI| = s \leq d,k$, and $\Delta < 1/s$. Then, for all $0 < \gamma \leq 1$, there exists a $1$-bounded, $(2/s)$-strictly monotone symmetric $\gI$-sparse Boolean function $g_\gI : \{0,1\}^T \rightarrow \RR$ and Transformer head parameters such that $f_\tfScalar(X(b); W_Q, W_K, W_V, W_C, w = e_1)$ $(\gamma/4)$-uniformly approximates $g_\gI$. The norms satisfy
\[ \norm{W_Q}_F \leq \frac{\log\pa{\frac{8T}{\gamma}}}{1-s\Delta}, \qquad \norm{W_K}_F \leq s, \qquad \norm{W_V}_F \leq 2, \qquad \norm{W_C}_F \leq 1. \]
Also, there exists a $1$-bounded, $2$-injective $\gI$-sparse Boolean function $g'_\gI : \{0,1\}^T \rightarrow \RR^s$ and $s$-headed Transformer parameters such that $f_\tfHead\pa{ X(b); \left\{ W_Q^{[h]}, W_K^{[h]}, W_V^{[h]}, W_C^{[h]} \right\}_{h=1}^s } \circ \Pi_s$ uniformly approximates $g'_\gI$. The norms of each head satisfy
\[ \norm{W_Q^{[h]}}_F \leq \frac{\log\pa{\frac{8T}{\gamma}}}{1-s\Delta}, \qquad \norm{W_K^{[h]}}_F \leq 1, \qquad \norm{W_V^{[h]}}_F \leq 2, \qquad \norm{W_C^{[h]}}_F \leq 1. \]
\end{lemma}
\begin{lemma}[Trainable $P$, no MLP]
\label{lem:bool-trainable-pos}
Suppose $X(b) = P + E_b$ with trainable $P$. Let $\gI \subseteq [T]$ such that $|\gI| = s \leq d,k$. Then, for any $0 < \gamma \leq 1$, and with the same $g_\gI$ as in Lemma~\ref{lem:bool-deterministic-pos}, there exists $P$ and Transformer head parameters such that $f_\tfScalar(X(b); W_Q, W_K, W_V, W_C, w = e_1)$ $(\gamma/4)$-uniformly approximates $g_\gI$. The norms satisfy
\[ \norm{P^\top}_{2,1} \leq s, \qquad \norm{W_Q}_F \leq \log\pa{\frac{8T}{\gamma}}, \qquad \norm{W_K}_F \leq 1, \qquad \norm{W_V}_F \leq 2, \qquad \norm{W_C}_F \leq 1. \]
Also, for the same $g'_\gI$ as in Lemma~\ref{lem:bool-deterministic-pos}, there exists $P$ and $s$-headed Transformer parameters such that $f_\tfHead\pa{ X(b); \left\{ W_Q^{[h]}, W_K^{[h]}, W_V^{[h]}, W_C^{[h]} \right\}_{h=1}^s } \circ \Pi_s$ uniformly approximates $g'_\gI$. The norms of each head satisfy
\[ \norm{P^\top}_{2,1} \leq s, \qquad \norm{W_Q^{[h]}}_F \leq \log\pa{\frac{8T}{\gamma}}, \qquad \norm{W_K^{[h]}}_F \leq 1, \qquad \norm{W_V^{[h]}}_F \leq 2, \qquad \norm{W_C^{[h]}}_F \leq 1. \]
\end{lemma}
\begin{lemma}[Bag of vectors, no MLP]
\label{lem:bool-bag}
Suppose $X(b) = V + \diag(b)$. Let $\gI \subseteq [T]$ such that $|\gI| = s \leq d,k$, and $\Delta < 1/s$. Then, for all $s\Delta < \gamma < 1$, there exists an $s$-bounded, $(1/s)$-strictly monotone symmetric $\gI$-sparse Boolean function $g_\gI : \{0,1\}^T \rightarrow \RR$ and Transformer head parameters such that $f_\tfScalar(X(b); W_Q, W_K, W_V, W_C, w = e_1)$ $(\gamma/4)$-uniformly approximates $g_\gI$. The norms satisfy
\[ \norm{W_Q}_F \leq \frac{ \log\pa{ \frac{8Ts(1+\Delta)}{\gamma - s\Delta} } }{1-s\Delta}, \qquad \norm{W_K}_F \leq s+1, \qquad \norm{W_V}_F \leq 2s, \qquad \norm{W_C}_F \leq s. \]
Also, there exists a $1$-bounded, $(1/s)$-injective $\gI$-sparse Boolean function $g'_\gI : \{0,1\}^T \rightarrow \RR^s$ and Transformer head parameters such that $f_\tfHead(X(b); W_Q, W_K, W_V, W_C) \circ \Pi_s$ uniformly approximates $g'_\gI$. The norms satisfy the same bounds as above.
\end{lemma}
\begin{lemma}[Monotone to symmetric functions via MLP]
\label{lem:bool-symmetric}
Let $f : \{0,1\}^T \rightarrow \RR$ be any real-valued symmetric $s$-sparse Boolean function with index set $\gI$.
Let $W_Q, W_K, W_V, W_C$ be the parameters of a function
\[ f_\tfHead(X; W_Q, W_K, W_V, W_C) := \softmax\pa{v_\CLS^\top W_Q W_K^\top X^\top} X W_V W_C, \]
and let $\Pi_1 : \RR^d \rightarrow \RR$ be the projection onto the first coordinate. Suppose that under some mapping $b \mapsto X(b)$, $f_\tfHead \circ \Pi_s$ $(\gamma/4)$-uniformly approximates a $B$-bounded $\gamma$-strictly monotone symmetric $\gI$-sparse Boolean function $g : \{0,1\}^T \rightarrow \RR$, for some $\gamma$. Then, there exists a function $f_\tfmlp \in \gF_\tfmlp$ with the same weights $W_Q, W_K, W_V, W_C$, and 3-layer feedforward network weights $W_1, W_2, w$, such that
\[f_\tfmlp(X(b)) = f(b), \qquad \forall b \in \{0,1\}^T,\]
with dimensions $(d_2,d_3) = (4(s+1), 2(s+1))$
and weight norms satisfying
\[\norm{W_1}_\infty \leq \frac{8 \max(1,B)}{\gamma}, \qquad
\norm{W_2}_\infty \leq \frac{8s}{\gamma}, \qquad
\norm{w}_\infty \leq \max_{b \in \{0,1\}^T} |f(b)|.
\]
\end{lemma}
\begin{lemma}[Injective to arbitrary functions via MLP]
\label{lem:bool-asymmetric}
Let $f : \{0,1\}^T \rightarrow \RR$ be any real-valued $s$-sparse Boolean function with index set $\gI$ such that $|\gI| = s \leq d$.
Let $W_Q, W_K, W_V, W_C$ be the parameters of a function
\[ f_\tfHead(X; W_Q, W_K, W_V, W_C) := \softmax\pa{v_\CLS^\top W_Q W_K^\top X^\top} X W_V W_C, \]
and let $\Pi_s : \RR^d \rightarrow \RR^s$ be the projection onto the first $s$ coordinates. Suppose that
under some mapping $b \mapsto X(b)$, $f_\tfHead \circ \Pi_s$ $(\gamma/4)$-uniformly approximates a $\gamma$-injective function $g : \{0,1\}^T \rightarrow \RR^s$ satisfying $\norm{g(b)}_\infty \leq B$.
Then, there exists a function $f_\tfmlp \in \gF_\tfmlp$ with the same weights $W_Q, W_K, W_V, W_C$, and 3-layer feedforward network weights $W_1, W_2, w$, such that
\[f_\tfmlp(X(b)) = f(b), \qquad \forall b \in \{0,1\}^T,\]
with dimensions $(d_2,d_3) = (4s2^s, 2 \cdot 2^s)$
and weight norms satisfying
\[\norm{W_1}_\infty \leq \frac{8 \max(1,B)}{\gamma}, \qquad
\norm{W_2}_\infty \leq \frac{8s}{\gamma}, \qquad
\norm{w}_\infty \leq \max_{b \in \{0,1\}^T} |f(b)|.
\]
\end{lemma}

\subsection{Useful lemmas}
We will use a construction which approximates a ``hard selection'' of $s$ indices using the softmax mixture; for this, we will need to quantify the approximation error when the inputs to the softmax function are bounded.

\begin{lemma}[Softmax truncation]
\label{lem:softmax-truncation}
Let $z \in (\RR \cup \{-\infty\})^T$ such that $z_t \geq R$ for each $1 \leq t \leq s$, and $z_t \leq 0$ for each $s+1 \leq t \leq T$. Define $z' \in (\RR \cup \{-\infty\})^T$ so that $z'_t = z_t$ for $1 \leq t \leq s$, and $z_t = -\infty$ for $s+1 \leq t \leq T$. Then, letting $e^{-\infty} = 0$ in the definition of $\softmax(\cdot)$, we have
\[ \norm{ \softmax(z') - \softmax(z) }_1 \leq 2\frac{T-s}{s\exp(R)} < \frac{2T}{\exp(R)}. \]
\end{lemma}
\begin{proof}
We have
\[\norm{ \softmax(z') - \softmax(z) }_1
=
\sum_{t=1}^s \exp(z_t)\pa{\frac{1}{\mathbf{1}^\top \exp(z')} - \frac{1}{\mathbf{1}^\top \exp(z)}}
+ \sum_{t=s+1}^T \frac{\exp(z_t)}{\mathbf{1}^\top \exp(z)}.\]

The first summation is equal to
\[ 1 - \frac{\mathbf{1}^\top \exp(z')}{\mathbf{1}^\top \exp(z)} \leq \frac{T-s}{s\exp(R)}, \]
while the same upper bound holds for the second summation, since each term is at most $\frac{1}{s \exp(R)}$.
\end{proof}

Our results on approximating arbitrary sparse Boolean functions will depend on a generic construction for robustly approximating an arbitrary function $f : \RR^d \rightarrow \RR$ with a feedforward neural network. For simplicity of presentation, we use a standard\footnote{For example, this follows from the discussion in Chapter 4 of \citep{nielsen2015neural}.} $3$-layer ReLU network construction, which \emph{exactly} represents a piecewise constant function in specified regions.

\begin{lemma}[Exact function representation with a 3-layer ReLU net]
\label{lem:3-layer-relu}
Let $f : \RR^{d_f} \rightarrow \RR$, and let $x_1, \ldots, x_n \in \RR^{d_f}$ such that $\norm{x_i}_\infty \leq B$ for each $i \in [n]$, $\norm{x_i - x_j}_\infty \geq 4\delta$ for each $i \neq j \in [n]$. Then, there is a 3-layer feedforward network with ReLU activations, with parameters $W_1 \in \RR^{(d_f+1) \times d_2}, W_2 \in \RR^{(d_2+1) \times d_3}, w \in \RR^{d_3}$\footnote{Here, $W_1, W_2$ have bias terms; $w$ does not.}, such that
\[ f_\mlp(x_i + z) = f(x_i) \]
for all $i \in [n]$ and $\norm{z}_\infty \leq \delta$, where $\mathsf{ReLU}(x) := x_+ = \max(0,x)$ is applied entrywise, with
\[ d_2 = 4nd_f, \quad d_3 = 2n,\]
\[\norm{W_1}_\infty \leq \frac{\max(1,B)}{\delta}, \quad \norm{W_2}_\infty \leq \frac{d_f}{\delta}, \quad \norm{w}_\infty \leq \max_{i\in[n]} |f(x_i)|. \]
\end{lemma}
\begin{proof}
First, we construct a one-dimensional ``bump'' function basis, and propagate the Lipschitz constants.
A threshold function with a linear ``ramp'' of width $\delta$ can be obtained from a linear combination of 2 ReLU functions:
\[\nu_\delta(x) := (x/\delta + 1)_+ - (x/\delta)_+
= \begin{cases}0 & x \leq -\delta\\
x/\delta + 1 & -\delta \leq x \leq 0\\
1 & x \geq 0\\ \end{cases}.\]
Next, we construct the bump function 
\[\psi_\delta(x) := \nu_\delta(x) - \nu_\delta(2\delta - x).\]
By this construction, we have $\psi_\delta(x) = 1$ for $0 \leq x \leq 2\delta$ and $\psi_\delta(x) = 0$ for $x \leq -\delta$ and $x \geq 3\delta$, interpolating linearly on $[-\delta,0]$ and $[2\delta,3\delta]$. Next, define
\[\psi_\delta(x;x_0) := \psi_\delta(x-x_0+\delta)\]
\[= \pa{\frac{x - x_0}{\delta} + 2}_+
- \pa{\frac{x - x_0}{\delta} + 1}_+
- \pa{\frac{x_0 - x}{\delta} + 2}_+
+ \pa{\frac{x_0 - x}{\delta} + 1}_+
\]
so that $\psi_\delta(x;x_0) = 1$ for $|x - x_0| \leq \delta$, $\psi_\delta(x;x_0) = 0$ for $|x - x_0| \geq 2\delta$.

We construct the first layer $W_1 \in \RR^{(d+1) \times (4nd)}$ using these bump functions: indexing the $4nd$ dimension by $(h \in [4],i \in [n],j \in [d])$, we construct
\[ [W_1]_{:,(1,i,:)} := \begin{bmatrix}
\frac{1}{\delta} I \\
-\frac{x_i}{\delta} + 2\cdot\mathbf{1}^\top
\end{bmatrix}, \qquad [W_1]_{:,(2,i,:)} := \begin{bmatrix}
\frac{1}{\delta} I \\
-\frac{x_i}{\delta} + \mathbf{1}^\top
\end{bmatrix},\]
\[ [W_1]_{:,(3,i,:)} := \begin{bmatrix}
-\frac{1}{\delta} I \\
\frac{x_i}{\delta} + 2\cdot\mathbf{1}^\top
\end{bmatrix}, \qquad [W_1]_{:,(4,i,:)} := \begin{bmatrix}
-\frac{1}{\delta} I \\
\frac{x_i}{\delta} + \mathbf{1}^\top
\end{bmatrix},\]
so that
\[\pa{ [x \; 1]^\top \bra{ [W_1]_{j,(1,i,:)} \, [W_1]_{j,(2,i,:)} \, [W_1]_{j,(3,i,:)} \, [W_1]_{j,(4,i,:)}} }_+ [1 \; -\!1 \; -\!1 \; 1]^\top = \psi_\delta(x;[x_i]_j). \]

The second layer is used to construct $n$ activations which are indicators of whether $x$ is in the neighborhood of each $x_i$. For each $x_i$, we will simply average the $d_f$ one-dimensional indicators for each coordinate, and implement a threshold function $\nu_{\delta/d_f}(1 - x)$.
We choose $W_2 \in \RR^{(4nd_f+1) \times (2n)}$, with the $4nd_f+1$ dimension indexed by $(h,i,j)$ and an extra bias dimension $\bot$, and the $2n$ dimension indexed by $(h' \in \{1,2\},i' \in [n])$ so that
\[[W_2]_{(h,i,:),(h',i')} := [1\; -\! 1 \; -\! 1 \; 1]_h \cdot \mathbbm{1}[i=i'] \cdot \frac{1}{\delta} \cdot \mathbf{1},\]
\[[W_2]_{\bot,(1,i')} := 1 - \frac{d_f}{\delta}, \qquad [W_2]_{\bot,(2,i')} := - \frac{d_f}{\delta}.\]
Finally, the third (output) layer $w \in \R^{2n}$, with dimensions indexed by $(h \in \{1,2\},i \in [n])$, multiplies the indicators of each $x_i$ by the desired $f(x_i)$:
\[w_{(1,i)} := f(x_{i}), \quad w_{(2,i)} := -f(x_{i}).\]

For any $x_0 \in \RR^{d_f}$, let $B_\delta(x_0)$ be the set of $x$ such that $\norm{x-x_0}_\infty \leq \delta$. By this construction, for each $x \in B_\delta(x_i)$, we have $f(x) = x_i$, as required.
\end{proof}

Note that we use 3-layer ReLU networks for function approximation in order to minimize the introduction of unnecessary notation. Some minor remarks:
\begin{itemize}
    \item It would be routine to replace this construction with any architecture which can represent an arbitrary function \emph{approximately} \citep{hornik1989multilayer,cybenko1989approximation}; this includes the 2-layer feedforward networks (and nonlinear activations other than the ReLU) which are typically used by Transformers in practice.
    \item It is possible to embed this construction in $f_\tfmlp$ with a 2-layer ReLU network, by using $(W_C, W_1, W_2)$ and introducing a nonlinearity after $W_C$, without changing the results.
    \item When $d_f = 1$, $W_2$ is unnecessary (one can represent $f$ directly using the bump function basis).
\end{itemize}

\subsection{Beyond Boolean domains}
\label{subsec:appendix-beyond-boolean}

These representational results are stated with a Boolean domain $\{0,1\}^T$ for clarity and simplicity of presentation; this is not essential or fundamental to these constructions. Generalizations to the following input domains are straightforward:

\begin{itemize}
    \item Discrete categorical tokens $\{1, \ldots, M\}^T$. Use $M$ orthogonal token embeddings, instead of only $2$. The ``sparse function approximation'' construction uses $sM^s$ instead of $s2^s$ parameters. The smaller ($\mathrm{poly}(s)$-parameter) representation of symmetric Boolean functions has no clear analogue.
    \item Continuous inputs on a bounded domain (e.g. $[0, 1]^T$). The (fixed or trainable) positional embeddings can still select the $s$ sparse indices. Replace the discrete ReLU network construction with any continuous universal function approximation construction. The ``bag-of-vectors'' formulation has no clear analogue.
\end{itemize}

The capacity results from Section~\ref{sec:capacity} hold for any input domain, as long as the embeddings are bounded in the $\lVert \cdot \rVert_{2,\infty}$ norm. Note that Transformers are predominantly used with discrete inputs.

\subsection{Proofs}
\label{subsec:appendix-approx-proofs}

Throughout the constructions in each case, we will refer to standard coordinate bases in several spaces:
\begin{itemize}
    \item $E_0, E_1 \in \RR^d$ denote the embeddings of the $0, 1$ tokens $E_{0,:}, E_{1,:}$.
    \item $e_1^{(k)}, \ldots, e_k^{(k)}$ denotes the standard basis in $\RR^k$.
    \item $e_i^{(d)}$ denotes the standard basis in $\RR^d$.
    \item $e_1^{(T)}, \ldots, e_T^{(T)}, e_\CLS^{(T)}$ denotes the standard basis in $\RR^{T+1}$ with the special $\CLS$ index.
    \item Recall that the $v_i$ form a $\Delta$-approximate orthonormal basis for $\RR^d$, $v_\CLS, e_0, e_1$ are exactly orthogonal to each of them as well as each other, and $d$ is chosen such that these conditions can be met.
\end{itemize}
Let $n(i)$ be a unique bijection between $\mathcal{I}$ and $[s]$. Let $v_\gI := \sum_{i\in\gI} v_i$.

\paragraph{Approximate vector equality.} We will use $u \approx_\eps v$ to denote that two vectors $u,v \in \RR^{d_u}$ satisfy $\norm{u - v}_\infty \leq \eps$.

\begin{proof}[Proof of Lemma~\ref{lem:bool-deterministic-pos}]
We construct attention heads such that the softmax mixture always selects the indices in $\gI$.

\textbf{Single head, deterministic $P$.} We seek to approximate the $1$-bounded, $(2/s)$-strictly monotone function
\[\frac{1}{s} \sum_{i \in \gI} \chi_i,\]
where $\chi_i = +1$ if $b_i = 0$ and $-1$ if $b_i = 1$. Set
\[W_Q := Rv_\CLS e_1^{(k)\top}, \quad W_K := v_\gI e_1^{(k)\top}, \quad W_V := (E_0 - E_1) e_1^{(k)\top}, \quad W_C := e_1^{(k)} e_1^{(d)\top},\]
where $R$ will be chosen later. Then, by approximate orthogonality,
\[v_\CLS^\top W_Q W_K^\top X^\top = v_\CLS^\top W_Q W_K^\top (P + E_b)^\top = v_\CLS^\top W_Q W_K^\top P^\top \approx_{Rs\Delta} R \sum_{i \in \gI} e^{(T)\top}_i. \]
By Lemma~\ref{lem:softmax-truncation},
\[\norm{ \softmax\pa{ v_\CLS^\top W_Q W_K^\top X^\top } - \frac{1}{s} \sum_{i \in \gI} e^{(T)\top}_i }_1 \leq \frac{2T}{\exp(R - 2Rs\Delta)}. \]
Finally, we have
\[XW_VW_C = E_b W_V W_C = \pa{ \sum_{i \in [T]} \chi_i e_i^{(T)} } e_1^{(d)\top},\]
so that by H\"older's inequality,
\[ f_\tfHead(X) \circ \Pi_1 = \softmax\pa{v_\CLS^\top W_Q W_K^\top X^\top} X W_V W_C e_1^{(d)}
\approx_{ \frac{2T}{\exp(R - 2Rs\Delta) }} \frac{1}{s} \sum_{i \in \gI} \chi_i. \]
To get $(\gamma/4)$-uniform approximation, we choose
\[R = \frac{ \log\pa{ \frac{8T}{\gamma} } }{1-s\Delta}.\]

\textbf{Multiple heads, deterministic $P$.} For $h = 1, \ldots, s$, and the same $R$ as above:
\[W_Q^{[h]} := Rv_\CLS e_1^{(k)\top}, \quad W_K^{[h]} := v_{n^{-1}(h)} e_1^{(k)\top}, \quad W_V^{[h]} := (E_0 - E_1) e_2^{(k)\top}, \quad W_C^{[h]} := e_1^{(k)} e_{h}^{(d)\top}.\]
This is the same construction as above, but each head only selects one of the coordinates in $\gI$. Thus, by the same analysis,
\[ f_\tfHead(X) \circ \Pi_s \approx_{\frac{2T}{\exp(R - 2Rs\Delta) }} \sum_{i \in \gI} \chi_i e_{n(i)}^{(d)}. \]
This function is clearly $1$-bounded and $2$-injective.
\end{proof}

\begin{proof}[Proof of Lemma~\ref{lem:bool-trainable-pos}]
The constructions closely follow Lemma~\ref{lem:bool-deterministic-pos}, but are simpler.

\textbf{Single head, trainable $P$.} For each $i \in \gI$, set the trainable positional embeddings to be
\[P_{i,:} := \begin{cases}v_1 & i \in \gI \\ 0 & \text{otherwise} \end{cases}.\]
Set
\[W_Q := Rv_\CLS e_1^{(k)\top}, \quad W_K := v_1 e_1^{(k)\top}, \quad W_V := (E_0 - E_1) e_1^{(k)\top}, \quad W_C := e_1^{(k)} e_1^{(d)\top}.\]
Now, we have (with equality)
\[v_\CLS^\top W_Q W_K^\top X^\top = R \sum_{i \in \gI} e^{(T)\top}_i,\]
so that Lemma~\ref{lem:softmax-truncation} gives
\[\norm{ \softmax\pa{ v_\CLS^\top W_Q W_K^\top X^\top } - \frac{1}{s} \sum_{i \in \gI} e^{(T)\top}_i }_1 \leq \frac{2T}{\exp(R)}. \]
Like before, we have
\[ f_\tfHead(X) \circ \Pi_1 = \softmax\pa{v_\CLS^\top W_Q W_K^\top X^\top} X W_V W_C e_1^{(d)}
\approx_{ \frac{2T}{\exp(R) }} \frac{1}{s} \sum_{i \in \gI} \chi_i. \]
To get $(\gamma/4)$-uniform approximation, we choose
\[R = \log\pa{ \frac{8T}{\gamma} }.\]

\textbf{Multiple heads, trainable $P$.} For each $i \in \gI$, set the trainable positional embeddings to be
\[P_{i,:} := \begin{cases}e_{n(i)}^{(d)} & i \in \gI \\ 0 & \text{otherwise} \end{cases}.\]
For $h = 1, \ldots, s$, and the same $R$ as above:
\[W_Q^{[h]} := Rv_\CLS e_1^{(k)\top}, \quad W_K^{[h]} := e_{h}^{(d)} e_1^{(k)\top}, \quad W_V^{[h]} := (E_0 - E_1) e_1^{(k)\top}, \quad W_C^{[h]} := e_1^{(k)} e_h^{(d)\top}.\]
This is the same construction as above, but each head only selects one of the coordinates in $\gI$. Thus, by the same analysis,
\[ f_\tfHead(X) \circ \Pi_s \approx_{\frac{2T}{\exp(R) }} \sum_{i \in \gI} \chi_i e_{n(i)}^{(d)}. \]
\end{proof}

\begin{proof}[Proof of Lemma~\ref{lem:bool-bag}]
This input mapping does not use position embeddings, and does not need multiple heads to implement arbitrary (non-symmetric) functions. The constructed monotone and injective functions are slightly different, but the proof strategy is very similar. The key difference is that the softmax mixture is uniform only on the positions $i \in \gI$ where $b_i = 1$.

\textbf{Bag of vectors, scalar output.}
The function we will approximate is defined as follows:
\[g_\gI(r) := \frac{r-s}{r+1}, \quad \text{where } r = \sum_{i \in \gI} b_i, \quad s = |\gI|. \]
Note that this function is $(1/s)$-strictly monotone, and has absolute value bounded by $s$.
Set
\[W_Q := Rv_\CLS e_1^{(k)\top}, \quad W_K := (v_\gI + v_\CLS) e_1^{(k)\top},\]
\[W_V := \sum_{i \in \mathcal{I}} v_i e_{n(i)}^{(k)\top} - v_\CLS \pa{ \sum_{i \in \mathcal{I}} e_{n(i)}^{(k)} }^\top, \quad W_C := \sum_{i \in \mathcal{I}} e_{n(i)}^{(k)} e_1^{(d)\top},\]
where $R$ will be chosen later. Then, by approximate orthogonality,
\[v_\CLS^\top W_Q W_K^\top X^\top \approx_{Rs\Delta} R \pa{ v_\CLS + \sum_{i \in \gI} b_i e^{(T)\top}_i }, \]
so that by Lemma~\ref{lem:softmax-truncation},
\[\norm{ \softmax\pa{ v_\CLS^\top W_Q W_K^\top X^\top } - \frac{1}{r+1} \pa{ e_\CLS^{(T)\top} + \sum_{i \in \gI} b_i e^{(T)\top}_i } }_1 \leq \frac{2T}{\exp(R - 2Rs\Delta)}. \]
Finally, we have
\[XW_VW_C e_1^{(d)} = - s e_\CLS^{(T)} + \sum_{i \in [T]} b_i \cdot v_i^\top v_\gI \cdot e_i^{(T)}
\approx_{s\Delta} - s e_\CLS^{(T)} + \sum_{i \in \gI} b_i e_i^{(T)},\]
so that
\begin{align*}
&| f_\tfHead(X) \circ \Pi_1 - g_\gI(r) | \leq \\
&\qquad\qquad \norm{ \softmax\pa{ v_\CLS^\top W_Q W_K^\top X^\top } - \frac{1}{r+1} \pa{ e_\CLS^{(T)\top} + \sum_{i \in \gI} b_i e^{(T)\top}_i } }_1 \pa{\norm{XW_VW_C e_1^{(d)} }_\infty + s\Delta} \\
&\qquad\qquad\qquad + \norm{ \softmax\pa{ v_\CLS^\top W_Q W_K^\top X^\top } }_1 (s\Delta)\\
&\leq \frac{2Ts(1+\Delta)}{\exp(R - 2Rs\Delta)} + s\Delta.
\end{align*}
To get $(\gamma/4)$-uniform approximation, we choose
\[R = \frac{ \log\pa{ \frac{8Ts(1+\Delta)}{\gamma - s\Delta} } }{1-s\Delta}.\]

\textbf{Bag of vectors, $s$-dimensional output.}
We use the same construction as above, except
\[ W_C := \sum_{i \in \mathcal{I}} e_{n(i)}^{(k)} e_{n(i)}^{(d)\top}. \]
This will allow us to approximate the function
\[g'_\gI(b) = \frac{1}{r+1} \sum_{i \in \gI} (b_i - 1) e_{n(i)}^{(d)}, \]
which is $(1/s)$-injective and has absolute value is bounded by $1$.
Then, for each $i \in \mathcal{I}$, we have
\[XW_VW_C e_i^{(d)} = - e_\CLS^{(T)} + v_i^\top v_\gI \cdot e_i^{(T)}
\approx_{s\Delta} - e_\CLS^{(T)} + b_i e_i^{(T)}.\]
Repeating the above analysis for each coordinate, we have
\[ f_\tfHead(X) \circ \Pi_s \approx_\eps g'_\gI(r), \]
where a slightly tighter bound
\[\eps = \frac{2T(1+s\Delta)}{\exp(R - 2Rs\Delta)} + s\Delta\]
comes from the fact that $\norm{XW_VW_C e_i^{(d)}}_\infty$ is now bounded by $1$ instead of $s$. The previous choice of $R$ suffices for $(\gamma/4)$-uniform approximation.
\end{proof}

\begin{proof}[Proof of Lemma~\ref{lem:bool-symmetric}]
This follows by instantiating Lemma~\ref{lem:3-layer-relu} with $\delta = \gamma/8, d_f=1, n=s+1$. Notice that a $(\gamma/4)$-uniform approximation of a $\gamma$-strictly monotone function satisfies the conditions needed for Lemma~\ref{lem:3-layer-relu}.
\end{proof}

\begin{proof}[Proof of Lemma~\ref{lem:bool-asymmetric}] This follows by instantiating Lemma~\ref{lem:3-layer-relu} with $\delta = \gamma/8, d_f=s, n=2^s$. Notice that a $(\gamma/4)$-uniform approximation of a $\gamma$-injective function satisfies the conditions needed for Lemma~\ref{lem:3-layer-relu}.
\end{proof}

%% file: C-full-experiments.tex
\section{Details for experiments}

\subsection{Empirical scaling laws (Figure~\ref{fig:sparse-and-scaling-laws})}
\label{subsec:appendix-sparse-and}
In this section, we provide details for the sample complexity scaling law experiments, which are the main empirical verification of the $\log T$ dependence of the sample complexity arising from the analysis.

\paragraph{Data.} Synthetic supervised learning tasks corresponding to learning a 3-sparse conjunction were generated by the protocol described in the main paper, parameterized by sample size $m$ and context $T$: in each trial, one of the $\binom{T}{3}$ subsets of indices was selected uniformly at random\footnote{Note that due to the permutation-symmetry of the Transformer architecture (as long as the position embeddings are initialized with a permutation-symmetric distribution), it is equivalent to select $\gI = [s]$. Also, by symmetry of the initial token embeddings, $\mathsf{AND}$ and $\mathsf{OR}$ are interchangeable in these experiments and results.}, under i.i.d. Bernoulli inputs $x_i \sim \mathrm{Bern}(p)$. Here, $p = (1/2)^{1/3}$ was chosen so that the classes are balanced $\Pr[y = 0] = \Pr[y = 1]$. $m$ samples were drawn this distribution to form a training set (rejecting training sets which were compatible with multiple hypotheses), and $10^4$ samples were drawn from the same distribution as a holdout validation set. The grid of problem instances was constructed as follows: $T \in \{100, 150, 200, \ldots, 750, 800\} \cup \{900, 1000, 1100\}$, $m \in \{50, 60, 70, \ldots, 200\}$.

\paragraph{Training.} A 1-layer Transformer network (with a scalar output at a special trainable token $x_\CLS$ at an extra index \CLS) was trained with Adam \citep{kingma2014adam} and full-batch gradients of the cross entropy loss for binary classification.  40 independent trials were performed (re-randomizing the dataset generation); each trial was restarted (with fresh random initialization and dropout masks) 5 times before being labeled as a failure. Cross-validation was performed on a holdout sample of size $10^4$ every 10 iterations. At the end of $1000$ training iterations, the trial was counted as a success if the maximum validation accuracy throughout training was greater than $0.99$. (In $100\%$ of runs, the training loss was driven to $<10^{-4}$, with $100\%$ training accuracy, within $1000$ iterations.)

\paragraph{Architecture.} Like \citep{chen2021decision}, our experimental setup is based on a popular PyTorch implementation (\url{https://github.com/karpathy/minGPT}), with some optimizations for faster $1$-layer training and inference. This implementation includes widely-used architectural details (GeLU activations; dropout) which were not discussed in the theoretical analysis; refer to the referenced repository for details. All hyperparameter settings left undiscussed are taken from the default settings in this codebase.

\paragraph{Hyperparameters.} A fixed architecture was used ($d = 64, k = 4$, 16 parallel heads), with trainable positional embeddings initialized with Gaussian entries $\gN(0, \sigma^2)$, $\sigma = 0.02$, $3$ input token embeddings (corresponding to $0, 1, \CLS$), and $2$ output embeddings (corresponding to the possible labels $0,1$). For regularization mechanisms, typical choices were used: $0.1$ for \{attention, embedding, output\} dropout; $10^{-4}$ weight decay. The Adam optimizer was instantiated with typical parameters $\eta = 10^{-3}, \beta_1=0.9, \beta_2=0.999$.

\paragraph{Infrastructure and computational costs.} All experiments were performed on an internal cluster, with NVIDIA Tesla P100, NVIDIA Tesla P40, and NVIDIA RTX A6000 GPUs. Each training run took at most 10 minutes (with most computation time spent on cross-validation), for a total of $\sim 150$ GPU-hours.

\subsection{Other figures}
\label{subsec:appendix-sparse-parity}

\paragraph{Example training curves.} Figure~\ref{fig:sparse-and-extra} \emph{(left)} shows the best training curves (in terms of highest validation accuracy) out of 5 restarts for all 40 replicates, in the settings $T=300, m=200$ and $T=300, m=50$. As the sample size $m$ decreases, the trained models overfit with higher probability; when they overfit, they differ significantly in validation accuracy.

\paragraph{Attention weights.} In Figure~\ref{fig:sparse-and-extra}, the normalized attention weights at the \CLS position are shown for a Transformer model trained for 500 iterations with $T = 50, m = 300$ (achieving $100\%$ validation accuracy on $10^4$ holdout samples), for 100 validation samples (which each induce different attention weights, shown in the scatter plot). One key difference (for simplicity of visualization) is that this network only has one attention head, with embedding dimensions $d = k = 64$.

\paragraph{Parity.} These experiments use the same architecture as the main $\mathsf{AND}$ experiments, except the. In the loss curves presented in Figure~\ref{fig:sparse-parity}, gradient-based training is done on streaming online losses (so that there is no training/validation split), with batch size 2048.

%% file: D-related.tex
\section{Additional related work}
\label{appendix-sec:related}

In this section, we discuss some additional related work.

\paragraph{Domains beyond language.} We provide a few references on the successful application of Transformers to domains beyond natural language processing; this frontier is continually and rapidly evolving.
\begin{itemize}
    \item With minimal changes compared to the analogous natural language settings, Transformer sequence models have been applied to theorem proving \citep{polu2020generative} and program synthesis \citep{chen2021evaluating}.
    \item Beyond sequence modeling: In self-supervised learning of image representations, the Vision Transformer (ViT) \citep{dosovitskiy2020image} has sometimes outperformed older convolution-based architectures, particularly when pretrained on massive datasets. Further architectural variants have been proposed for vision and other continuous modalities \citep{tolstikhin2021mlp,lee2021fnet,jaegle2021perceiver,jaegle2021perceiver2,d2021convit}.
    \item Natural sciences: A state-of-the-art pipeline for protein structure prediction \citep{jumper2021highly} features a self-attention-based component.
    \item Reinforcement learning: Transformer architectures have shown promise for planning \citep{janner2021reinforcement} and offline RL  \citep{chen2021decision}.
\end{itemize}

\paragraph{Expressive power of Transformers.}
Several works establish results on the representational power of self-attention architectures in regimes where the statistical guarantees are necessarily weak or vacuous (i.e. there are too many functions in the class). \cite{dehghani2018universal,yun2019transformers,bhattamishra2020ability,bhattamishra2020computational} establish universal function approximation and Turing-completeness, which have been known for previous architectures \citep{siegelmann1995computational}. Our work is a significantly finer-grained analysis, in which we establish a hierarchy of function classes (indexed by sparsity $s$) representable by these architectures, with tight (in terms of $T$) statistical guarantees. \cite{hron2020infinite,yang2020tensor} analyze properties of the kernels induced by Transformers at the infinite-width limit. Quoting the discussion
  following Theorem~4.1 of \citep{wei2021statistically}, who show statistically meaningful approximations of TM using Transformers: \emph{``The
    correct norm-based Rademacher complexity bound to use for
    Transformers is unclear.''}
The connection between Transformers and circuits also appears in \citep{elhage2021mathematical, olsson2022context}, with a different technical approach (interpreting the classes of computations performed by attention weights and heads).
\cite{likhosherstov2021expressive} analyze the sparsity patterns representable by a self-attention head, with results superficially similar to ours: when the embedding dimension is at least logarithmic in the context length, all sparse matrices can be approximately realized by an attention head. However, their analysis is not about the capacity of the function class: it quantifies over the input $X$, and holds the parameters $(W_Q, W_K, \ldots)$ to be constant (rather than vice versa). This finding serves as an interesting complement to our result: even though the attention mixture weights can take on exponentially many sparsity patterns for distinct inputs, the generalization error scales as $\log(T)$.

\paragraph{Interpreting attention mixtures.} A line of empirical work (``BERTology'') has made progress on understanding and interpreting state-of-the-art Transformer language models by examining the activations of their attention mechanisms \citep{clark2019does,tenney2019bert,rogers2020primer}. In some cases, these works have found instances in which Transformers seem to have learned features that are reminiscent of (sparse) hand-crafted features used in natural language processing, without explicit supervision. Our analysis formalizes the intuition that self-attention heads can represent sparse interactions within the context in a statistically meaningful way.

\paragraph{Other theoretical work on self-attention.} \citet{kerg2020untangling} analyze self-attention and its benefits for learning long-term dependencies by establishing gradient norms bounds and showing how attention helps address the problem of gradient vanishing in recurrent networks. In contrast to our results that analyze  the statistical and representational properties of attention-based architectures, this work focuses on the computational aspects of gradient-based methods on recurrent networks with self-attention.

\paragraph{Other attention-based architectures.} Our analysis is amenable to computationally efficient variants of the Transformer which use parsimonious (e.g. low-rank) approximations of the softmax kernel, like the Performer \citep{choromanski2020rethinking}. Building upon the success of modern attention-based architectures, a large body of work (e.g. \citet{goyal2020recurrent, goyal2021neural}, and the works cited within) has sought to design architectures which induce model sparsity and modularity. Our analysis is relevant to any architecture that uses a softmax (or similar) bottleneck for statistical capacity, and could inform design principles for norm-based capacity control of these architectures.

\paragraph{Synthetic experiments with Transformers.} \citet{power2021grokking} train small Transformer networks on synthetic algebraic tasks, and discover an abrupt phase transition from overfitting to correct generalization similar to ours. \cite{tay2020long} propose some synthetic tasks for benchmarking the ability of Transformer variants to capture long-range dependences. \cite{chen2021decision} present a synthetic demonstration of extrapolation (inferring a maximum-reward path from random walk observations) when using Transformers for offline reinforcement learning. \cite{lu2021pretrained} probe the transfer learning capabilities of pretrained Transformers, and consider some simple Boolean tasks. Our experimental protocol of learning sparse Boolean functions provides a simple and fundamental setting for elucidating computational and statistical properties of sequence modeling architectures.